%% file: arxiv.tex
\documentclass[11pt]{article}

\usepackage{microtype}
\usepackage{graphicx}
\usepackage{booktabs} 
\RequirePackage{natbib}
\usepackage[margin=1in]{geometry}

\usepackage{tikz}
\usetikzlibrary{automata,decorations.markings,positioning,arrows}
\usepackage{wrapfig, amsmath, amsthm, mathtools}

\usepackage{makecell}

\usepackage{caption, subcaption}

\usepackage[utf8]{inputenc} 
\usepackage[T1]{fontenc}    
\usepackage{xcolor}
\usepackage[colorlinks,citecolor=blue,urlcolor=blue,linkcolor=blue,linktocpage=true]{hyperref}

\usepackage{url, float}            
\usepackage{amsfonts}       
\usepackage{amssymb}
\usepackage{nicefrac}

\usetikzlibrary{shapes,decorations,arrows,calc,arrows.meta,fit,positioning}
\tikzset{
    -Latex,auto,node distance =1 cm and 1 cm,semithick,
    state/.style ={ellipse, draw, minimum width = 0.7 cm},
    point/.style = {circle, draw, inner sep=0.04cm,fill,node contents={}},
    bidirected/.style={Latex-Latex,dashed},
    el/.style = {inner sep=2pt, align=left, sloped}
}

\makeatletter

\usepackage{libertine}
\usepackage{libertinust1math}
\usepackage[T1]{fontenc}

\usepackage{parskip}

\usepackage[utf8]{inputenc} 
\usepackage[T1]{fontenc}    
\usepackage{hyperref}       
\usepackage{url}            
\usepackage{booktabs}       
\usepackage{amsfonts}       
\usepackage{nicefrac}       
\usepackage{microtype}      
\usepackage{subcaption}
\usepackage{url}
\usepackage{graphicx}
\usepackage{multirow}
\usepackage{tabularx}
\usepackage{amsmath}
\usepackage{amssymb}
\usepackage{dsfont}
\usepackage{natbib}
\usepackage{dsfont}
\usepackage{color}
\usepackage{xcolor}
\usepackage{multirow}
\usepackage[normalem]{ulem}
\usepackage{outlines}
\usepackage{centernot}

\newcommand{\argmax}{\text{argmax}}
\usepackage{stmaryrd}
\usepackage{algorithm}
\usepackage[normalem]{ulem}
\newcommand{\squigglypath}[1]{\hspace{1mm} \raisebox{0.45em}{\uwave{\hspace{0.5cm}}} \hspace{1mm}}
\usepackage[noend]{algpseudocode}

\newcommand{\BlackBox}{\rule{1.5ex}{1.5ex}}  
\newtheorem{example}{Example} 
\newtheorem{theorem}{Theorem}
\newtheorem{lemma}[theorem]{Lemma} 
\newtheorem{proposition}[theorem]{Proposition}

\newcommand{\Ecal}{\mathbb{E}}
\newcommand{\Xcal}{\mathcal{X}}

\newcommand{\Rcal}{\mathbb{R}}
\newcommand{\Dcal}{\mathcal{D}}

\newcommand{\Gcal}{\mathcal{G}}
\newcommand{\Ical}{\mathcal{I}}

\newcommand{\sign}{\text{sgn}}
\newcommand{\pa}{\text{pa}}
\newcommand{\anc}{\text{anc}}

\newcommand{\var}{\mathop{\rm var}}

\newcommand{\ind}{\mathds{1}}
\newcommand{\indep}{\mathrel{\perp\mspace{-10mu}\perp}}
\newcommand{\notindep}{\centernot{\indep}}
\newcolumntype{L}[1]{>{\hsize=#1\hsize\raggedright\arraybackslash}X}%

\newenvironment{proofsketch}{\par\noindent{\bf Proof Sketch\ }}{\hfill\BlackBox\\[2mm]}

\newcommand{\edit}[1]{{#1}}

\title{Discovering Optimal Scoring Mechanisms\\
in Causal Strategic Prediction}

\author{Tom Yan, Shantanu Gupta, Zachary C. Lipton \\
       Carnegie Mellon University\\
       \{\href{mailto:tyyan@cmu.edu}{tyyan}, \href{mailto:shantang@cmu.edu}{shantang}, \href{mailto:zlipton@cmu.edu}{zlipton}\}@cmu.edu}

\begin{document}

\maketitle

\begin{abstract}
\input{sections/00_abstract}
\end{abstract}

\section{Introduction}
\input{sections/10_intro}

\section{Related Works}
\input{sections/20_related_works}

\section{Causal Strategic Prediction}
\label{sec:csp_framework}
\input{sections/25_problem}

\section{General Graphs} \label{sec:tradeoff}
\input{sections/40_tradeoff}

\section{Discovery Algorithms with Linear number of Deployments}
\label{sec:linear_time_alg}
\input{sections/30_discovery}

\section{Discovery Algorithm under General Costs}
\label{sec:quad_time_alg}
\input{sections/35_general_discover}

\section{Discussion}
\input{sections/50_discussion}

\newpage

\bibliographystyle{abbrvnat}
\bibliography{paper}

\newpage

\appendix

\input{appendix}

\end{document}

%% file: sections/00_abstract.tex
Faced with data-driven policies, individuals will manipulate their features 
to obtain favorable decisions. 
While earlier works cast these manipulations as undesirable gaming, 
recent works have adopted a more nuanced causal framing in which 
manipulations can improve outcomes of interest, 
and setting coherent mechanisms requires accounting 
for both predictive accuracy and \edit{improvement of the outcome.}
Typically, these works focus on 
\emph{known} causal graphs, consisting only of an outcome and its parents. 
In this paper, we introduce a general framework 
in which an outcome and $n$ observed features 
are related by an arbitrary \emph{unknown} graph 
and manipulations are restricted 
by a fixed budget and cost structure.
We develop algorithms that leverage strategic responses 
to discover the causal graph in a finite number of steps.
Given this graph structure, we can then derive mechanisms that trade off between accuracy and improvement.
Altogether, our work deepens links between causal discovery and incentive design and provides a more nuanced view of learning under causal strategic prediction.

%% file: sections/10_intro.tex
In consequential settings, 
machine learning models do more than predict. They also drive decisions that impact people's lives.
For example, credit scores may simultaneously 
serve as predictions of the likelihood of repayment
and as the basis on which loans are approved. 
When decisions impact individuals
whose features are manipulable,  
these individuals will be \emph{incentivized} 
to intervene on their features
in order to raise the model scores. 
Whether or not these increases in score 
(e.g., predicted likelihood of repayment)
result in improvements in the outcome of interest
(e.g., actual likelihood of repayment)
depends on the \emph{causal} relationships between 
the features and the outcome.
Thus, this causal knowledge is crucial to designing scoring \emph{mechanisms} that serve as both accurate predictors and beneficial incentives.

A blossoming line of research on strategic machine learning studies these incentive effects~\citep{bruckner2011stackelberg, hardt2016strategic, dong2018strategic, kleinberg2020classifiers, bechavod2020causal, levanon2022generalized, zhang2021incentive, sundaram2021pac, ahmadi2021strategic, yan2022margin, ghalme2021strategic, chen2020learning, brown2022performative, perdomo2020performative, mendleranticipating}.
\citet{hardt2016strategic}
conceive of feature manipulations as gaming,
putting aside the possibility that manipulations
might change the outcome of interest.
More recently, researchers have recognized
that manipulations can causally influence
the outcome interest, and seek to learn optimal scoring mechanisms for outcome improvement 
\citep{kleinberg2020classifiers, shavit2020causal}. 
However, most works thus far assume that the underlying graph structure 
is known and consists only of the outcome node and its parents.
A notable exception is \citet{miller2020strategic} who demonstrate 
that producing an optimal scoring mechanism is at least as hard
as identifying the underlying causal graph. 
However, they do not explore how the ability to deploy mechanisms
and observe the induced strategic responses can be leveraged
to efficiently identify the underlying causal structure, 
and in turn, derive the optimal scoring mechanism.

\subsection{Paper Contributions}

We introduce the general framework of \emph{Causal Strategic Prediction} (CSP),
where variables are related by a Structural Causal Model (SCM)~\citep{pearl2009causality} associated with an arbitrary causal graph and a firm interacts 
with a population of individuals over a sequence of turns.
Notably, our framework can model complex relationships among the $n$ features,
and between the features and the outcome,
capturing more of the dynamics that may be at play in real-world scenarios.
In our setting, the firm initially has no knowledge of the underlying graph.
Then, over a sequence of turns, the firm iteratively chooses scoring mechanisms 
and then observes data resulting from the distribution
induced as individuals play their best response.
Our main contribution is a set of algorithms
that efficiently (in a number of turns proportional to the number of variables)
discover the underlying graphs and, in turn, identify the optimal mechanism.
Finally, we derive insights on tradeoffs between risk and improvement 
that arise when general graphs are used 
to model relationships between features and the outcome.
These insights motivate the necessity of considering general graphs \edit{with arbitrary graph structure}, which are captured in our framework.

In our setup, we model feature manipulations as soft interventions
on the underlying causal graph (Section~\ref{sec:csp_framework}).
Subject to some cost structure, individuals 
apply additive perturbations to variables,
which influence both the value of the intervened-upon variable
and all downstream variables in the graph 
(possibly, but not necessarily, including the outcome of interest). 
Capturing the causal effect of feature manipulation 
lies at the heart of strategic ML and recourse literature. 
Our framework allows us to quantify the causal effect of such changes, 
distinguishing the good (improvement) from the bad (gaming). 
Using this framework, we then derive tradeoffs 
between predictive performance and improvement, 
when general graphs are used to model the causal effects of feature manipulation (Section~\ref{sec:tradeoff}). 
Our analysis shows that such a tradeoff does not exist
in stylized graphs studied in most prior works,
but does exist in general graphs.
We uncover the source of this tradeoff, 
which reveals a notable insight challenging convention: 
anti-causal features (proxies) may not only be accurate predictors, 
\emph{but also} beneficial incentives.
    
We develop the first set of discovery algorithms 
that can identify \emph{arbitrary} graphs 
using the best responses from individuals 
(Sections~\ref{sec:linear_time_alg} and~\ref{sec:quad_time_alg}).
The crux of these algorithms is to use aptly chosen mechanisms 
to induce strategic responses, which can be used for causal discovery.
For a class of heterogeneous cost functions 
that generalizes separable quadratic costs~\citep{shavit2020causal},
we use a per-node incentivization strategy
to discover the graph in  $n$ rounds, 
where $n$ denotes the number of nodes in the graph 
(Algorithm~\ref{alg:quadratic_alg}).
In the linear cost setting, where identifying the graph may not be possible,
we develop an algorithm to nonetheless 
recover the Pareto-optimal scoring mechanisms (Algorithm~\ref{alg:linear_alg}).
Finally, for a broader class of cost functions 
that generalizes both quadratic~\citep{shavit2020causal} 
and linear costs~\citep{bechavod2020causal, kleinberg2020classifiers},
we develop an algorithm
that can discover additive graphs in at most $n(n-1)/2$ rounds
(Algorithm~\ref{alg:general_alg}).

Our paper concludes with conceptual insights that emerge from our investigation. 
We briefly touch upon the two main ones here. 
One insight is the possibility of using incentives to probe and understand causal structure. Our work introduces a new mechanism that enables induced distribution shift, typically viewed as a challenge to be overcome in domain adaptation literature~\citep{lipton2018detecting,zhang2013domain,magliacane2018domain}, could be useful for causal discovery. Our work could also prove useful to social scientists for whom incentives may merely be a means towards the end of discovering causal structure. Another insight is that proxies can boost both predictive accuracy and beneficial incentivization. This qualitative insight adds novel nuance to the conversation on when it is useful to \edit{include} anti-causal features in ML models.

%% file: sections/20_related_works.tex
\subsection{Causal View of Strategic Manipulation}

To our knowledge, \citet{kleinberg2020classifiers} are 
the first in the strategic ML literature to raise awareness 
that scoring mechanisms, when made transparent 
to strategic individuals,
can be leveraged to induce best responses that improve
an outcome of interest.
In subsequent work, \citet{shavit2020causal} study the optimization of improvement and accuracy under quadratic cost and a linear SCM with known graph structure, consisting of the outcome node $Y$ and its parents. Under the same linear SCM and quadratic cost setup, \citet{harris2022strategic} study recovering the SCM parameters by viewing the deployed model as an instrument and recover the SCM using two-stage least squares. Finally, with a linear SCM but a linear cost function, \citet{bechavod2020causal} study online learning of the SCM parameters
via strategic responses. By contrast, our work provides complementary algorithms that discover the initially unknown graph.

Addressing general graphs, \citet{miller2020strategic} 
connect strategic ML to causality by proving that incentive design can be 
at least as hard as causal discovery. 
Their hardness result states that in a static setting, access to an oracle that can set the cost function and return a model inducing improvement can be used to do causal discovery. By contrast, we consider a sequential setting, in which one does not have the oracle but is allowed to iteratively set mechanisms and observe the resultant distributions. We develop algorithms that adaptively set mechanisms to generate data that enables causal discovery. This sequential setting is also studied by~\cite{perdomo2020performative}, which is focused only on predictive accuracy and not both improvement and accuracy under a causal lens.

In summary, our results relax assumptions made in prior works on the graph structure (allowing for arbitrary graphs), the form of the SCM (allowing for non-linear structural equations), and cost functions (going beyond linear and quadratic costs). 
In Appendix~\ref{sec:game_vs_improve}, 
we demonstrate how our framework 
subsumes several canonical strategic ML settings. 
Importantly, unlike prior works, our results demonstrate
that indirect intervention via incentives schemes
can be leveraged both to reveal causal structure
where it might be otherwise unidentified 
and to derive improved mechanisms.

\subsection{Causal Discovery} 

Many works in the causal discovery literature 
focus on using observational data to identify the graph
up to Markov equivalence~\citep{spirtes2000causation, chickering2002optimal, kalisch2007estimating}.
Another line of research focuses 
on identifying the exact graph using observational data
under stronger assumptions on the functional forms 
of the structural equations 
and noise distributions~\citep{hoyer2008nonlinear, zhang2012identifiability, shimizu2014lingam, peters2014causal, uemura2022multivariate}.
By contrast, our work shows that by leveraging induced shifts from strategic best responses, 
we can identify the exact causal graph under comparatively weaker assumptions. 

Several prior works have also looked to 
characterize a narrower equivalence class
by leveraging interventional data 
\citep{eberhardt2007interventions, he2008active, hauser2012characterization, yang2018characterizing}.
Other works have 
characterized the interventional equivalence class
when the targets of the interventions are unknown 
\citep{jaber2020causal, squires2020permutation}.
By contrast, in our setting,
we discover the underlying graph in cases 
where soft interventions are carried out
at unknown targets by individuals
in response to the deployed scoring mechanisms.
%


%% file: sections/25_problem.tex
\begin{figure}[t]
    \centering
   \includegraphics[width=0.6\textwidth]{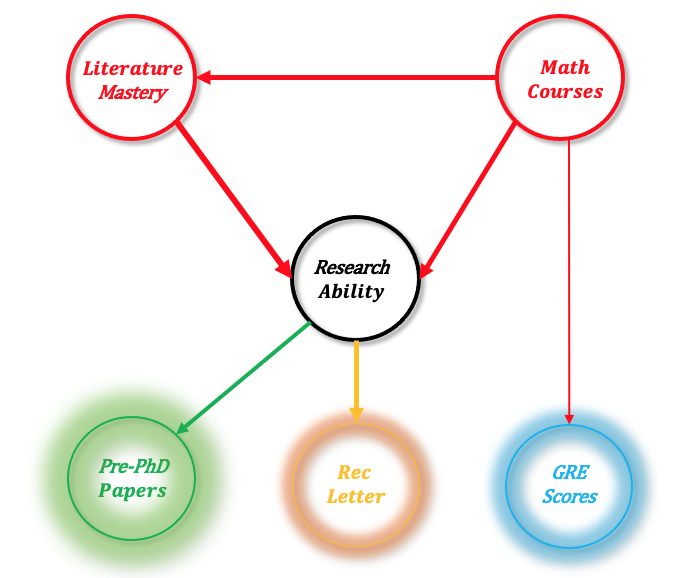}
    \caption{\edit{Above is a toy-example graph depicting causal relationships between five observable features considered in PhD admissions (``Literature Mastery'' may be observed in the applicant's statement of purpose); the edge thickness indicates the size of the causal effect, and the glow around the node represents the size of the exogenous noise. Through this graph, we highlight two key aspects of our setup (1) We consider general graphs with arbitrary graph structure, and importantly, proxies of the variable of interest $Y$ (``Research Ability''). Consideration of proxies have been largely absent in prior works, and yet two of the most important factors considered in PhD admissions, ``Recommendation Letters'' and ``Pre-PhD Papers'', are proxies of $Y$, making it crucial to consider such nodes in the modeling and thus the causal graph. (2) We model the causal-relationships between features. In admissions, schools often recommend extensive coursework in math, due to its numerous positive downstream effects. These causal effects are reflected in the red edges of the graph: indeed, studying math leads to not only downstream improvement in technical problem solving (``Research Ability''), but also in mathematical maturity for understanding prior works (``Literature Mastery'') as well as in reasoning ability for GRE test-taking (``GRE Scores'').}}
    \label{fig:admissions_graph}
\end{figure}

To motivate our setup, consider the following example: 
an university admissions committee seeks to design a mechanism $\hat{Y}$ (admission criteria) 
for admitting students based their measured attributes $X$ (as seen in their application profiles). 
The committee aims to admit the students with the best research ability $Y$.
In our setup, we consider the research ability $Y$ to be latent (but realized), and only observable ex-post; for example, admitted students' research abilities are later directly
observed by their advisors and via the research they produce. Note that if $Y$ is observable, which is typically not the case in admissions, hiring and various other selection processes, one may simply select using $Y$ without needing a predictive model of $Y$.

One natural goal for this mechanism is that it should accurately predict $Y$.
At the same time, the committee might wish to set a coherent set of incentives,
such that students optimizing their chances of acceptance 
would, in so doing, improve their actual research abilities. To some extent, these goals can be at odds:
the prediction-optimal mechanism may not be improvement-optimal,
and vice versa.
Moreover, whatever the committee's aims, determining the optimal mechanism
requires accounting for the impacts of incentives. 
While the committee might hope that the optimal mechanism
may be identified from the observational data, recently, \citet{miller2020strategic} demonstrated 
that even determining the improvement-optimal mechanism
is tantamount to discovering the exact causal graph,
a task that is not, in general, possible from observational data alone.


Now suppose that, accepting its ignorance,
the committee nevertheless adopts some heuristic 
for choosing a new mechanism, parameterized by $\theta_1$,
and thus induces a new distribution $\Dcal(\theta_1)$.
In the next period, the committee again chooses a new mechanism $\theta_2$,
inducing a subsequent distribution $\Dcal(\theta_2)$.
In each time period, the new mechanism induces 
a change in the distribution, 
revealing key details about 
the structure of the problem to the committee. We formalize and study such an interactive setup,
demonstrating how the committee may strategically choose a sequence of mechanisms,
such that the resulting sequence of induced distributions
suffice to exactly identify the causal graph,
and, in turn, reveal the set of Pareto-optimal mechanisms.
Notably, our methods can identify the exact causal graph
even in cases where the graph is not identifiable 
from observational data alone.

\paragraph{Key Features of the Setup} Before moving onto formal definitions,
we emphasize two features, \edit{illustrated in Figure~\ref{fig:admissions_graph}}, that distinguish 
our setup from many of those in prior works.

\begin{itemize}
    
    \item We use general graphs to model the relationship between features and, in particular, consider downstream proxies of the latent variable $Y$. Indeed, many criteria considered in admissions, hiring, and other selection processes like sports drafts are downstream of $Y$. For example, in admissions, recommendation letters are a proxy for, and not cause of, research ability; in hiring, interview performances are a proxy for, and not a cause of, on-the-job performance. 
    
    Thus, we view our consideration of general graphs not as generality for generality's sake, but as a crucial aspect that \emph{need} to be included in the modeling. \edit{Indeed, this aspect has been largely overlook in prior works, which focus exclusively on ancestors of $Y$ (the two red nodes in Figure~\ref{fig:admissions_graph}).} As we will see, the consideration of general graphs will require a more nuanced analysis of the tradeoff that arises between accuracy and improvement.

    
    \item Individuals take into account the full causal graph capturing inter-feature relationships when responding strategically; in particular, the best response accounts for all of the downstream effects of interventions. For instance, spending time on math courses may improve logical reasoning (and downstream from this, research ability) as well as GRE math scores. On the hand, expending time practicing GRE math sections may constitute only an improvement on the GRE score itself.
    
\end{itemize}

\paragraph{Setup} At each time step $t$, an institution 
releases a scoring mechanism $f_t = f(\cdot; \theta_t)$, $f_t: \Rcal^n \rightarrow \Rcal$. $f_t$ is used to both predict a real-valued, outcome interest of $Y \in \Rcal$ 
and to influence the individuals subject to $f_t$.
An individual with feature $X = x \in \Rcal^n$ will then best respond to $f_t$,
leveraging their knowledge 
of the causal graph
to choose a manipulated feature $\tilde{X} = \tilde{x}$ that maximizes 
the score under $f_t$,
subject to a cost structure. 
This feature manipulation will also alter the outcome of interest from $Y = y$ to $\tilde{Y} = \tilde{y}$.
At the end of the round, the institution observes only the induced distribution, $(\tilde{X}, \tilde{Y}) \sim \Dcal(\theta_t)$.
Absent manipulation, the data follows 
the \emph{natural distribution} 
$\Dcal_0$, i.e. $(X, Y) \sim \Dcal_0$.

\paragraph{Objectives} We consider two desiderata for $f(\cdot; \theta)$:

\begin{enumerate}
    \item The risk of $f$ on the induced distribution, $\Ecal_{(\tilde{X}, \tilde{Y}) \sim \Dcal(\theta)}[(f(\tilde{X}; \theta) - \tilde{Y})^2]$, which we seek to minimize.

    \item The improvement (i.e. causal effected on $Y$) induced by $f$, $\Ecal_{(\tilde{X}, \tilde{Y}) \sim \Dcal(\theta)}[\tilde{Y}] - \Ecal_{(X, Y) \sim \Dcal_0}[Y]$, which we seek to maximize. Note that this is equivalent to maximizing $\Ecal_{(\tilde{X}, \tilde{Y}) \sim \Dcal(\theta)}[\tilde{Y}]$.

\end{enumerate}

\subsection{Modeling the Causal Effects of Best Response}

To model how feature manipulations affect $Y$, 
we assume there is an unknown, underlying SCM 
(associated with a directed acyclic graph (DAG)) 
that captures the relationships between the variables.
The associated DAG has a directed edge $X_i \rightarrow X_j$,
if $X_i$ is a direct cause of $X_j$. 
The nodes in the graph consist of endogenous nodes $(X_1, ..., X_n, Y)$ along with corresponding exogenous nodes $(U_1, ..., U_n, U_y)$. Let $\pa(i)$ denote the indices of parent nodes of $i$. Each node $X_i$ is related to its parents $X_{\pa(i)}$ (used to denote $\{X_j: j \in \pa(i)\}$ for brevity)
by a structural equation with an arbitrary function $g_i \in C^1$:
\begin{equation}
\quad X_i = g_i(X_{\pa(i)}, U_i), \forall i \in [n].
\end{equation}
Like~\citet{miller2020strategic}, we make several assumptions on the causal graph to facilitate our analysis.
First, we assume causal sufficiency (no unobserved common causes of 
the endogenous nodes), and knowledge of the skeleton of the causal graph
(the set of undirected edges of the graph).
This is a mild assumption as, 
under certain faithfulness assumptions,
existing structure learning algorithms 
can be used to obtain this skeleton~\citep{scanagatta2019survey}.
Also, unless otherwise stated, 
we assume that the SCM is an Additive Noise Model~\citep{peters2017elements} (ANM). 
ANMs are broad class of models,
in which $g_i(X_{\pa(i)}, U_i) = g'_i(X_{\pa(i)}) + U_i$ for some function $g'_i$. 
Note in particular that ANMs generalize linear SCMs, 
which are commonly studied in prior strategic ML works~\citep{kleinberg2020classifiers,shavit2020causal,bechavod2020causal,harris2022strategic}.

\paragraph{Individuals are aware of the causal graph} To our knowledge, prior strategic ML works assume individuals best respond without accounting for causal structure, i.e., optimizing $f(x + a; \theta)$. That is, it is implicitly assumed that the underlying causal graph is such that there are no causal relationship between features. This assumption is less realistic when there is an underlying 
causal graph linking the features. Hence, similar to \citet{karimi2020algorithmic}, 
we model individuals as being causally aware and optimizing $f(\cdot; \theta)$ subject to the causal structure. 

This assumption of causal knowledge is in line with the standard assumption of information asymmetry in principal-agent models: the individuals (agents) know more than the institution (principal)~\citep{holmstrom1979moral, grossman1992analysis}. 
In prior strategic machine learning settings, agents know the true features (prior to manipulation), which are unknown to the principal. In our setting, we additionally assume that agents have causal knowledge, which is initially unknown to the firm.

\paragraph{Modeling feature manipulation}
We choose to model feature manipulations as soft interventions, which differs from prior works that model manipulations as being hard interventions~\citep{karimi2020algorithmic, karimi2021algorithmic}. Soft interventions are more suitable as a model of feature manipulation in the strategic ML context; it has the salient advantage of capturing downstream effects of interventions. For example, consider the graph $X_1 \rightarrow X_2$ and both $X_1, X_2$ are intervened upon. Under hard interventions, $X_2$ will not be affected by the change in $X_1$, while such a change would be  captured by soft interventions.

\subsection{Individual's Best Response}

We can now put everything together.
Faced with a scoring mechanism $f$,
an individual with realized exogenous variable $u$ 
and feature $x$ best responds
by solving for the optimal interventional values $a^* \in \Rcal^n$ as follows:

\begin{equation}
\label{eqn:best_response}
\begin{aligned}
a^* = \arg\max_{a} \quad & f(x'; \theta)\\
\textrm{s.t.} \quad & x'_j = g_j(x'_{\pa(j)}, u_j) + a_j, \quad \forall j \in [n],   \\
\quad & c(a_1, ..., a_n; x) \leq b,
\end{aligned}
\end{equation}
where $c$ is the individual's cost function 
(belonging to $C^1$) and $b$ is the budget. 
Under the optimal intervention $a^*$, the individual's $j$th feature changes from $x_j$ to $\tilde{x}_j = g_j(\tilde{x}_{\pa(j)}, u_j) + a^*_j$, defined recursively. 
In line with the standard strategic ML formulation, we assume 
that the individuals cannot directly intervene on $Y$, which can only be manipulated through $X$ \citep{miller2020strategic, shavit2020causal, kleinberg2020classifiers}. For instance, loan repayment likelihood may only be influenced through changing causal factors such as getting a higher-paying job.

\paragraph{Linear Graphs:} \edit{For a concrete example of the best response,} suppose that the SCM is linear and that
we may relate features $x$ to $u$ through an auto-regression matrix $B$ with an all-one diagonal: $x = Bu$ with $x_j = B_j^Tx_{\pa(j)} + u_j$ defined recursively coordinate-wise. Under manipulation, $x'_j = B_j^Tx'_{\pa(j)} + u_j + a_j = B_j^Tx'_{\pa(j)} + (u_j + a_j)$, or in vector form, $x' = B(u + a) = x + Ba$. Thus, if the mechanism is linear, $f(x') = w^Tx'$, we may explicitly write down the best response optimization program as: $\max_{a} w^T(x + Ba)$ s.t. $c(a_1, ..., a_n; x) \leq b$.

\edit{To further consolidate intuition, we will consider the following three-node linear, chain graph, which will help to reveal some of the intuition underlying the improvement induced by proxies. }

\begin{example}\label{example: two_node}
Consider the graph of $X_1 \rightarrow Y \rightarrow X_2$ 
with linear SCM: $Y = X_1 + U_Y$, $X_2 = \alpha_2 Y + U_2$ with $\alpha_2 > 0$, and all exogenous variables $U_{\cdot}$ have mean $0$.
\end{example}

\edit{Let the cost be $c(a) = a_1^2 + a_2^2$. Suppose $X_2$ is very predictive of $Y$. We wish to deploy $f = X_2$ and would like to know: what happens to the improvement for this $f$? A quick calculation yields that in the best response, the ratio between the intervention on $X_1$ and the intervention on $X_2$ is equal to $\alpha_2$. So if $\alpha_2$ is very small, then as we might expect for a proxy, the improvement is near zero since almost all of the budget goes into $a_2$ and intervening on $X_2$, which has no effect on changing $Y$. However, if $\alpha_2$ is very large, then $f = X_2$ will in fact be near-optimal improvement-wise despite $X_2$ being a proxy. This is because in the best response, almost all of the budget goes into $a_1$ and intervening on $X_1$, which in turn increases $Y$.}

\paragraph{Heterogeneous Cost:} In this framework, we allow the cost functions to be heterogeneous and dependent on the values of the features.
\edit{For instance, the cost of gaining greater expertise at different subjects may vary depending on a student's features, e.g., what they might be adept at doing.}
The only assumption we will make is that $c$ is bounded and satisfies a standard regularity condition in constrained optimization: $\nabla c \neq 0$ for $a$ on the surface $c(a_1, ..., a_n; x) = b$.

%% file: sections/40_tradeoff.tex
To begin our investigation, a key first question to address is whether there is even need to study strategic ML in the context of general graphs.
Is causal strategic prediction in general graphs no different from in simpler graphs that comprise of only $Y$'s ancestors and $Y$? If so, existing methods from strategic ML may already suffice.


\subsection{Tradeoffs under a General Graph Structure}

In this subsection, we establish that causal strategic prediction in general graphs introduces important considerations that do not arise in the simpler settings
where all features are ancestors of the outcome variable $Y$. 
In particular, when all features are ancestors of $Y$, 
no tradeoff arises between improvement and accuracy: 
there exists a model that is simultaneously improvement optimal and risk optimal.

By contrast, the two can be at odds in general graphs. 
The existence of a tradeoff thus requires us to reconcile 
the two objectives by solving for the Pareto front. As we show below, the source of this tradeoff is \emph{the descendants of $Y$}.
Our analysis establishes that incentivization is one setting, like prediction,
where \emph{anti-causal} features matter and should be considered.

\paragraph{All-ancestors graphs} 
In this setting, we observe that there is a mechanism, as a function of only the parents of $Y$, that is both improvement and risk optimal. We formalize this in the following proposition, whose proof may be found in Section~\ref{sec:tradeoff_proofs} of the Appendix.

\begin{proposition}
Model $f = g_Y(X_{\pa(Y)})$ maximizes improvement and minimizes risk.
\end{proposition}
This means that only local causal discovery (of the parents of $Y$) is needed for constructing an optimal model. Indeed, once $\pa(Y)$ is identified, one may learn $g_Y$ 
by deploying any $f(\cdot; \theta)$ and computing $\Ecal_{\Dcal(\theta)}[Y | X_{\pa(Y)}]$.
Since $Y$ cannot be intervened upon,
$a^*_Y = 0$ and $\tilde{y} = g_Y(\tilde{x}_{\pa(Y)}) +u_Y + a^*_Y =  g_Y(\tilde{x}_{\pa(Y)}) + u_Y$. Hence, $\Ecal_{\Dcal(\theta)}[Y | X_{\pa(Y)}] = g_Y(X_{\pa(Y)}) + \Ecal[U_Y]$ and we may obtain $g_Y$ by ignoring the offset. Interestingly, this establishes a link between robustness and improvement, 
and uncovers the potential applicability of existing methods that try to learn $g_Y$ across distributions with unknown interventions~\citep{peters2016causal,arjovsky2019invariant}. Finally, we note that under general graphs, $f = g_Y(X_{\pa(Y)})$ is in fact not improvement-optimal; an illustration of this may be found in Section~\ref{sec:tradeoff_proofs} of the Appendix.

\paragraph{How descendants of $Y$ results in tradeoffs} 
In general graphs where 
model features may be descendants of $Y$, there is a tradeoff.
The extent of this tradeoff is determined by the size of \emph{the variance of $U_Y$}.
On the one hand, descendants of $Y$ may be included in the model to predict $U_Y$ and attain a 
MSE lower than $\var(U_Y)$.
On the other hand, this inclusion of $Y$'s descendants 
may induce individuals to expend 
budget on descendants of $Y$, 
which does not lead to improvement. 

There is little tradeoff to be had when $\var(U_Y)$ is small,
since $f = g_Y(X_{\pa(Y)})$ can induce maximal improvement 
while attaining a small MSE of $\var(U_Y)$.
Thus, under causal sufficiency, a high $\var(u_Y)$ is a necessary condition for a large tradeoff to exist. However, as we will see in the example that follows, this condition is not a sufficient one. 

How much accuracy must be sacrificed for improvement optimality (or vice versa)
depends on (1) the SCM functions $g_i$,
which capture the strength of incentives;
(2) the variances of exogenous variables, 
which capture the predictiveness of features 
and how much they should be weighed 
in models with near-optimal accuracy. 
\edit{We will illustrate this tradeoff through Example~\ref{example: two_node}. In this setting, we have the following proposition, whose proof may be found in Section~\ref{sec:tradeoff_proofs} of the Appendix.}

\begin{proposition}
In Example~\ref{example: two_node}, there exists a SCM and cost structure, where the optimal improvement is $1$ and:

\begin{enumerate}
    \item Any mechanism with low risk $o(\epsilon)$ must have improvement at most $O(\epsilon)$
    
    \item There is a mechanism $f$, which is a function of only the proxy, that has low risk $o(\epsilon)$ and also a high improvement of at least $1 - \epsilon$.
\end{enumerate}

\end{proposition}


The two settings above show how the tradeoffs may be large or small depending the SCM parameters. 
This illustrates a main takeaway that proxies \emph{may be useful} not only for predicting $Y$, but also for incentivizing improvement. The usefulness of a proxy for purposes of prediction
hinges upon the noise level associated with the proxy. \edit{For example, in Figure~\ref{fig:admissions_graph}, `Recommendation Letters'' is depicted as being a more predictive feature than ``Pre-PhD Papers'', due to lower exogenous noise.} Its usefulness towards improvement hinges upon the cost structure:
anti-causal incentives are beneficial when the most efficient 
pathways to exploit for manipulating the variable
involve intervening upstream of the outcome of interest. \edit{In Figure~\ref{fig:admissions_graph}, `Recommendation Letters'' and ``Pre-PhD Papers'' may both be good incentives if it is costly to change these features without improving ``Research Ability''.}

These observations help to clarify why proxies are often used 
in real-world evaluation schemes. 
For example, recommendation letters are arguably downstream 
of demonstrated research ability.
However, recommendation letters are difficult 
for applicants to intervene upon directly. 
We might hope that the most efficient route to improving 
one's recommendation letters would be to intervene upstream, and
improve one's skills and performance in research.


\subsection{Computing the Pareto Front}

\begin{figure}[t]
    \centering
   \includegraphics[width=0.75\textwidth]{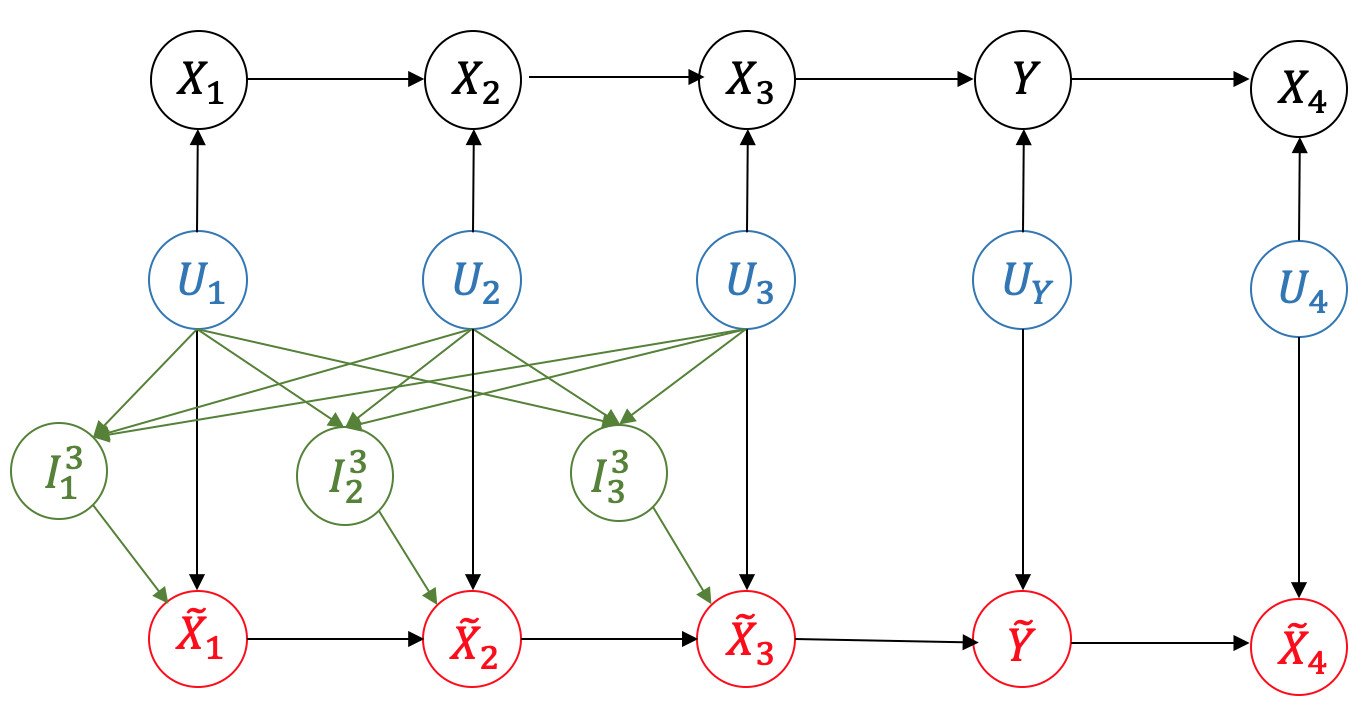}
    \caption{The above chain graph example illustrates confounding due to incentivized interventions. In black, we have four nodes in the unobserved graph $G$, representing features in $\Dcal_0$, with unobserved exogenous variables $U_i$ (blue). 
    In red, we have nodes in the observed graph $\tilde{G}$, 
    representing features in $\Dcal_3$. $\tilde{G}$'s structure and SCM mirrors that of $G$. The only difference is that $\tilde{G}$ is confounded by unknown intervention nodes $\{I_{j}^3\}_{j=1}^3$ (green), which open up a path between $\tilde{X}_1$ and $\tilde{X}_3$ that does not exist between $X_1$ and $X_3$ in $G$.}
    \label{fig:confounding_fig}
\end{figure}

Having established that there is a tradeoff between improvement and risk, it suffices to solve for the Pareto front. Our approach to solving for the Pareto front is as follows:

\begin{enumerate}
    \item Discover the causal graph.
    
    \item Use the graph structure to identify the SCM.
    
    \item With knowledge of the SCM, compute the Pareto front offline, \emph{without any further deployments,} by optimizing the following objective for any $\lambda > 0$:
    
    \begin{equation}
    \label{eqn:tradeoff}
    \begin{aligned}
    \min_{\theta} \Ecal_{(\tilde{X}, \tilde{Y}) \sim \Dcal(\theta)}[(f(\tilde{X}; \theta) - \tilde{Y})^2]  - \lambda \Ecal_{(\tilde{X}, \tilde{Y}) \sim \Dcal(\theta)}[g_Y(\tilde{X}_{\pa(Y)})].
    \end{aligned}
    \end{equation}
\end{enumerate}

Of the three steps, we note that the graph discovery is main challenge. The SCM may be efficiently identified with knowledge of the graph structure. For more details, please refer to Subsection~\ref{subsec:scm_identifiability} in the Appendix.

\paragraph{Causal Discovery} The key technical challenge we aim to solve in this paper is:

\begin{quote}
    How can we discover the graph using strategic best responses?
\end{quote}

We face three main challenges when discovering the graph. (1) Usually in causal discovery, one is able to designate the node(s) to intervene upon to obtain interventional data for discovery. In our case, we have to target nodes indirectly. (2) As we do not know the causal graph, we cannot observe which variables have been intervened upon as part of the best response optimization in Eq.~\ref{eqn:best_response}. (3) The resultant interventions confound the relationship between $\tilde{X}$'s: the conditional independences in $\Dcal_0$ may no longer hold in $\Dcal_i$ (illustrated in Figure~\ref{fig:confounding_fig}).

    
    


\paragraph{Performance Measure} As in~\citet{shavit2020causal}, our measure of algorithmic performance is the number of deployments needed before we have discovered the true graph. We note that another natural measure of algorithm performance may be the cumulative regret incurred during the discovery process, which we discuss further in  Appendix~\ref{sec:appendix_regret}.

\begin{table}[t!]
    \centering
    \begin{tabular}{lllll}
        \toprule
        \multirow{3}{*}{} & Algorithm~\ref{alg:quadratic_alg} & Algorithm~\ref{alg:linear_alg} & Algorithm~\ref{alg:general_alg}\\
        \midrule
        Cost Function &  Cost Function Class~\ref{cost_class:quadratic} & Linear Cost & Cost Function Class~\ref{cost_class:general} \\
        Applicable SCM   & General SCM   & Linear SCM & Additive SCM \\
        Number of Deployments   & $n$   & $\leq n$ & $\leq n(n-1) / 2$ \\
        \bottomrule
    \end{tabular}
\caption{Summary of the properties of the three algorithms. Note that Cost Function Class~\ref{cost_class:quadratic} generalizes quadratic costs. It is a subset of Cost Function Class~\ref{cost_class:general}, which generalizes both quadratic and linear costs. Additionally, we wish to highlight that for causal discovery, our algorithms only require that the individuals' cost functions belong to the respective cost function classes. The algorithms do not require knowledge of the cost functions. This knowledge is only needed for computing the Pareto front.}
\label{tab:alg_table}
\end{table}

\subsection{Assumptions}~\label{subsec:discovery_assumptions}

Before describing our discovery algorithms, we provide a succinct summary of the assumptions needed for discovering the causal graph and for computing the Pareto Front.

For the task of causal discovery, we develop three algorithms; please see Table~\ref{tab:alg_table}, which details the algorithms' performances and required assumptions. We wish to highlight that the algorithms only require that the individuals' cost functions belong to the listed cost function classes. Knowledge of the exact cost functions is not required.

For the task of computing the Pareto front, we require two assumptions:

\begin{enumerate}
    \item For identification of the SCM, as in \citet{miller2020strategic}, we require that the SCM lie in a broad class of SCM models: Additive Noise Model~\citep{peters2017elements} (ANM). Just as a recap, ANMs generalize linear SCMs, which have been commonly studied in prior works~\citep{kleinberg2020classifiers,shavit2020causal,bechavod2020causal,harris2022strategic}.

    \item For identifying the Pareto front, we require that the individuals' cost functions are known \textemdash a common assumption made in prior works \citep{shavit2020causal, bechavod2020causal}. Note however that, in certain settings such as the linear SCM and linear costs setting, Algorithm~\ref{alg:linear_alg} would not require knowledge of the individual cost function.

\end{enumerate}

%% file: sections/30_discovery.tex
\begin{algorithm}[t]
\caption{Discovery Algorithm using Per-Node Incentivization Strategy}
\label{alg:quadratic_alg}
\begin{algorithmic}[1]
\State Input: Distributions $\Dcal_0, \Dcal_1, ..., \Dcal_n$, Graph Skeleton $\text{GS}$
\State $\Gcal = \text{GS}$ \Comment{partially oriented graph $\Gcal$}
\State $\text{SG} = \{X_i\}_{i \in [n]} \cup \{Y\}$ \Comment{subgraph of unoriented nodes $\text{SG}$}
\State $S = \{\}$ \Comment{set of oriented nodes $S$ (complement of nodes in $\text{SG}$)}
\While{$|\text{SG}| > 1$}
\State Y\_root = True
\For{$X_i \in \text{SG} \setminus \{Y\}$}
    \State SG\_root = True
    \For{node $V \in \text{SG}$ adjacent to $X_i$ in $\text{GS}$} \Comment{for each adj node, test if $X_i$ is root}
        \State $S_{\text{control}} = \anc_{\Gcal}(X_i)$
        \If{$(V | X_i, S_{\text{control}})_{\Dcal_i} \neq (V | X_i, S_{\text{control}})_{\Dcal_0}$}\label{alg_line:natural_root_test} \Comment{$X_i$ cannot be a root in $\text{SG}$}
            \State SG\_root = False
            \State \textbf{break}
        \EndIf
    \EndFor
    \If{SG\_root} \Comment{each adj node is child as $X_i$ is root}
        \For{node $V \in \text{SG}$ adjacent to $X_i$ in $\text{GS}$}\label{alg_line:quad_orientation}
            \State Orient $X_i \rightarrow V$ in $\Gcal$, Remove edge $X_i - V$ from $\text{GS}$
        \EndFor
        \State $\text{SG} \leftarrow \text{SG} \setminus \{X_i\}$, $S \leftarrow S \cup \{X_i\}$ \Comment{update set of (un)oriented nodes $\text{SG}$ and $S$}
        \State Y\_root = False
        \State \textbf{break}
    \EndIf
\EndFor
\If{Y\_root} \Comment{no $X$ root node found in $\text{SG}$, by elimination, $Y$ must be the (only) root}
    \For{$X_j \in \text{SG}$ adjacent to $Y$ in $\text{GS}$} \Comment{each adj node is child as $Y$ is root}
        \State Orient $Y \rightarrow X_j$ in $\Gcal$, Remove edge $Y - X_j$ from $\text{GS}$
    \EndFor
    \State $\text{SG} \leftarrow \text{SG} \setminus \{Y\}$, $S \leftarrow S \cup \{Y\}$ \Comment{update set of (un)oriented nodes $\text{SG}$ and $S$}
\EndIf
\EndWhile
\State \textbf{return} $\Gcal$ \Comment{returns fully oriented graph}
\end{algorithmic}
\end{algorithm}

\subsection{Per-node Incentivization}

To begin, we examine the natural exploration strategy of per-node incentivization: deploying $\{f = X_i\}_{i=1}^n$. While stylized, this simple exploration strategy is useful in that it introduces a minimal amount of confounding. Since only one node is incentivized at a time, there is minimal confounding on the graph that results from induced interventions, and facilitates understanding the causal structure of the true graph through the confounded graph.
    
One may also observe that, besides the choices of $f$ determined by the exploration strategy, the cost function also matters, as it influences where the interventions occur. An arbitrary cost function such as $(b + 1)\mathds{1}\{\prod_{i \in [n]} a_i = 0\} + \mathds{1}\{\prod_{i \in [n]} a_i \neq 0\} c(a; x)$ will force an intervention on every node for any $f$, maximally confounding $\tilde{X}$: if there is one node which is not intervened upon, the cost would be $b + 1$ and would exceed the budget. And so, additional assumptions are needed on the cost functions to allow for efficient discovery.


\paragraph{Separable Heterogeneous Cost Functions}~\label{cost_class:quadratic} We identify a general class of cost functions that allows for efficient causal discovery. It is defined as follows:

$$c(a; x) = \sum_{i=1}^n c_i(a_i; x_{S_i}),$$
where $S_i \subseteq \anc(i)  \cup \{i\}$. Note that the cost is heterogeneous when $S_i \neq \emptyset$. Each cost function $c_i$ is assumed to be strictly increasing in the magnitude of $a_i$ and under zero change to $X_i$, the cost is zero: $c_i(0; x_S) = 0$. The key condition we will require of $c_i$ is that for each $i \in [n]$ and 
for all attainable values of $\{x_j\}_{j \in {\anc(i)  \cup \{i\}}}$, $\frac{\partial c(a; x)}{\partial a_i} \Bigr|_{a_i=0} = c'_i(0; x_{S_i}) = 0$. This key condition will ensure the property that node $i$ will be intervened upon in distribution $\Dcal_i$, which is the distribution induced by $f = X_i$.

An example of a cost function in this class is the (popular) homogenous quadratic cost $c(a) = \frac{1}{2} \|a\|^2$ (e.g., see~\citet{shavit2020causal}). For another example, under the chain graph $X_1 \rightarrow ... \rightarrow X_n$, an example heterogeneous cost function can be $c(a; x) = \sum_{i=1}^n c(a_i; x_1, .., x_i)$. 

\paragraph{Algorithm} With the guarantee that node $i$ is intervened upon in $\Dcal_i$, 
we develop causal discovery Algorithm~\ref{alg:quadratic_alg}, requiring access to the natural distribution $\Dcal_0$ and $\{\Dcal_i\}_{i=1}^n$, where $\Dcal_i$ is the resultant distribution when $f = X_i$ is deployed.
For facility of exposition, 
we will also assume access to the graph skeleton, which is assumed to be known in prior works~\citep{miller2020strategic}. Please see Appendix~\ref{sec:assumption_justification} for further discussions on these two assumptions.

In a nutshell, Algorithm~\ref{alg:quadratic_alg} is a top-down algorithm in which we iteratively discover root nodes in the current subgraph. The crux of the algorithm is a root-node test based on the observation that only a root node $V'$, controlling for all its ancestors $\anc(V')$, will be such that  each of its children $V''$ satisfy $(V'' | V', \anc(V'))_{\Dcal_i} = (V'' | V', \anc(V'))_{\Dcal_0}$. This is because all induced interventions (upstream) will be blocked. Identifying the root means that every node it is adjacent to in the subgraph must be a child. We may then orient accordingly and recurse on the remaining subgraph. Please see Section~\ref{sec:quad_cost_proofs} in the Appendix for more details.

\begin{theorem} \label{thm:quad_thm}
Algorithm~\ref{alg:quadratic_alg} recovers the full graph structure with $n$ deployments.
\end{theorem}

\begin{proofsketch}
We first characterize the confounding by the induced intervention. We show that in $\Dcal_i$, outside of $i$, only $i$'s ancestors may be intervened upon. In particular, none of $i$'s descendants are intervened upon in $\Dcal_i$.

With this, we then argue that, subject to a notion of parameter faithfulness, a node $X_i$ is a root of the current subgraph $SG$ iff $(V | X_i, \anc_{GS}(X_i))_{\Dcal_i} = (V | X_i, \anc_{GS}(X_i))_{\Dcal_0}$ for every node $V$ it is adjacent to (Condition~\ref{alg_line:natural_root_test}). This is because, controlling for its ancestors, roots of the subgraph will be such that each of its edge with its child will be unconfounded. Also, we show this will not hold for non-root nodes and in particular for the node's edge with its parent.

Thus, using Condition~\ref{alg_line:natural_root_test}, we may iteratively discover a $X$ root in the subgraph, orient its edges with its children and then recurse on the remaining subgraph. The only other case to handle is when no $X$ root is found. By process of elimination, $Y$ must be the only root in the subgraph and its edges in the subgraph are oriented accordingly.
\end{proofsketch}

\begin{algorithm}[t]
\caption{Optimization Algorithm under Linear SCM and Linear Cost}
\label{alg:linear_alg}
\begin{algorithmic}[1]
\State Deploy $f(x) = x_1$, $f(x)= -x_1$ \Comment{collect initial pair of distributions}
\State Compute $\Ecal_{\Dcal_0}[X] = (\Ecal_{\Dcal_1}[X] + \Ecal_{\Dcal_{-1}}[X]) / 2$
\State Initialize $S = \{\Dcal_1, \Dcal_{-1}\}$, $W = \{\Ecal_{\Dcal_1}[X] - \Ecal_{\Dcal_{-1}}[X]\}$
\For{$i = 2, ..., n$}\label{alg_line:all_linear_loop}
    \State Compute some $w_i$ in the nullspace of $W^T$ using SVD of $WW^T$ \Comment{$w_i$ is such that all prior interventions cannot change the score under $w_i$}
    \State Deploy $f(x) = w_i^Tx$, $f(x) = -w_i^Tx$ and obtain $\Dcal_i, \Dcal_{-i}$
    \If{$\Dcal_i \in S$} \Comment{encounter duplication} \label{alg_line:non_observe_linear}
        \State $W \leftarrow W \cup \{w_i\}$
    \Else \Comment{observe pair of new distributions with new underlying intervention}
        \State $S \leftarrow S \cup \{\Dcal_{i}, \Dcal_{-i}\}$
        \State $W \leftarrow W \cup \{ \Ecal_{\Dcal_i}[X] - \Ecal_{\Dcal_0}[X]\}$
        \State \textbf{if} $|S| = 2k:$ \textbf{break} \Comment{only $2k$ pairs of distinct distributions are possible}
    \EndIf
\EndFor
\State $\Pi = \{\}$ \Comment{initialize set of all (MSE, Improvement) pairs}
\For{$i = 1, ..., |S|$}\label{alg_line:all_linear_pareto}
    \State $\Dcal_i = S[i]$
    \State Solve the QP program: \Comment{computes lowest MSE attainable under a particular intervention} \begin{equation}\label{eqn:min_mse_given_intervention}
        \begin{aligned}
        \min_{w} \quad & R_{\Dcal_i}(w)\\
        \textrm{s.t.} \quad & w^T (\Ecal_{D_i}[X] - \Ecal_{\Dcal_0}[X]) \geq w^T (\Ecal_{D_j}[X] - \Ecal_{\Dcal_0}[X]) \quad \forall j \neq i
        \end{aligned}
        \end{equation}
    
    to obtain smallest attainable MSE, $R_i$.
    \State $\Pi \leftarrow \Pi \cup \{(\Ecal_{\Dcal_i}[Y], R_i)\}$ \Comment{add one pair of (MSE, Improvement)}
\EndFor
\State \textbf{return} all non-dominated pairs in $\Pi$
\end{algorithmic}
\end{algorithm}

\subsection{Adaptive Exploration Strategy under Linear SCM and Linear Cost} \label{subsec:linear_cost}

As shown in the case of Algorithm~\ref{alg:quadratic_alg}, in order to do discovery, per-node incentivization is such that for every node $i$, there exists some $f_i$ such that node $i$ will be intervened upon. However, this will not be true for all cost functions and then the per-node incentivization strategy will not work. It is instructive then to study settings where the cost is such that $\frac{\partial c(a; x)}{\partial a_i} \Bigr|_{a_i=0} \neq 0$ as in such cases, incentivizing $i$ will not necessarily lead to $i$ being intervened upon.

One such case is the setting of linear SCMs and linear cost functions. As is commonly assumed in recourse literature, we will allow for immutable features, i.e., a feature $i$ with $c_i = \infty$.
In fact, causal discovery may be impossible in this case. For example, when most features are immutable, we may only observe interventions on a small part of the graph, with the rest of the graph never intervened upon. 

Nevertheless, our ultimate goal is to obtain the Pareto front and causal discovery is only a means to end. We show that in this particular setting, it is possible to obtain the Pareto frontier of risk vs improvement, without causal knowledge. This case-study introduces the nuance that causal knowledge is sufficient, but not necessary for the tradeoff optimization. 
The key observation we will leverage is that there are $2k$ possible distributions that may be induced, where $k \leq n$ is the number of mutable features.

\paragraph{Algorithm} At a high level, Algorithm~\ref{alg:linear_alg} first induces all possible interventions by iteratively deploying models $w_i$ orthogonal to interventions seen in prior rounds. This is intuitive since $w_i$ will be such that the previous interventions will not change the score. If there is at least one unseen intervention that can increase the score, we will be able to observe a new intervention and the corresponding distribution. Now, the only time $w_i$ will not induce a new intervention is when all unseen interventions \emph{also} will not change the score under $w_i$. However, we prove that this event (Line~\ref{alg_line:non_observe_linear}) cannot happen too many times, and eventually we observe distributions corresponding to all inducible interventions.

Next, after inducing all possible interventions, for each possible intervention in the for-loop on line~\ref{alg_line:all_linear_pareto}, we compute the model that induces this intervention and attains the lowest MSE (QP Program~\ref{eqn:min_mse_given_intervention}). Since the improvement is fixed when the underlying intervention is fixed, we have just enumerated the set $\Pi$ of all possible (risk, improvement) pairs that \emph{could be} on the Pareto front. The rest of models are dominated risk-wise. Thus, it just remains to retain all non-dominated (risk, improvement) pairs in $\Pi$ to obtain the Pareto frontier. Please see Section~\ref{sec:lin_cost_proofs} in the Appendix for the full proof.

\begin{theorem}
\label{thm:linear_alg_thm}
Algorithm~\ref{alg:linear_alg} computes the Pareto-Frontier using at most $2n$ deployments.
\end{theorem}


Thus, through Algorithm~\ref{alg:linear_alg}, we show that we \emph{need not always} have to discover the causal structure in order to compute Pareto optimal models that serve as both accurate predictors and beneficial incentives.

%% file: sections/35_general_discover.tex
\begin{algorithm}[t]
\caption{Discovery Algorithm for Additive Graphs under General Cost}
\label{alg:general_alg}
\begin{algorithmic}[1]
\State Input: Distribution $\Dcal_0$, Graph Skeleton $\text{GS}$
\State $\Gcal = \text{GS}$ \Comment{partially oriented graph $\Gcal$}
\State $\text{SG} = \{X_i\}_{i \in [n]} \cup \{Y\}$ \Comment{subgraph of unoriented nodes $SG$}
\State $S = \{\}$ \Comment{set of oriented nodes $S$ (complement of nodes in $SG$)}
\While{$|SG| > 1$}
\State Y\_leaf = True
\For{$X_i \in \text{SG} \setminus \{Y\}$}
    \State $X_{P_i} \leftarrow$ nodes adjacent to $X_i$ in $\text{GS}$ \label{alg_line:general_regression_args}
    \State \textbf{if} $P_i = \emptyset$, $\hat{g}_i = 0$; \textbf{else} $\hat{g}_i = (X_i | X_{P_i})_{\Dcal_0}$
    \State Deploy $f' = X_i - \hat{g}_i(X_{P_i})$ to obtain distribution $\Dcal'$ \label{alg_line:general_new_deploy}
    \State X\_leaf = True
    \For{node $V \in \text{SG} \setminus \{X_i\}$} \Comment{test if $X_i$ is a leaf node in $SG$}
        \If{$\Ecal_{\Dcal'}[V] \neq \Ecal_{\Dcal_0}[V]$} \label{alg_line:general_test}
            \State X\_leaf = False
            \State \textbf{break}
        \EndIf
    \EndFor
    \If{X\_leaf}
        \For{node $V \in \text{SG}$ adjacent to $X_i$ in $\text{GS}$} \label{alg_line:general_orientation} \Comment{each adj node is parent as $X_i$ is leaf}
            \State Orient $V \rightarrow X_i$ in $\Gcal$, Remove edge $V - X_i$ from $\text{GS}$
        \EndFor
        \State $\text{SG} \leftarrow \text{SG} \setminus \{X_i\}$, $S \leftarrow S \cup \{X_i\}$ \Comment{update set of (un)oriented nodes $SG$ and $S$}
        \State Y\_leaf = False
        \State \textbf{break}
    \EndIf
\EndFor
\If{Y\_leaf} \Comment{no $X$ leaf node found in $SG$, by elimination, $Y$ must be the (only) leaf}
    \For{node $X_j \in \text{SG}$ adjacent to $Y$ in $\text{GS}$} \Comment{each adj node is parent as $Y$ is leaf}
        \State Orient $X_j \rightarrow Y$ in in $\Gcal$, Remove edge $X_j - Y$ from $\text{GS}$
    \EndFor
    \State $\text{SG} \leftarrow \text{SG} \setminus \{Y\}$, $S \leftarrow S \cup \{Y\}$ \Comment{update set of (un)oriented nodes $S$ and $SG$}
\EndIf
\EndWhile
\State \textbf{return} $\Gcal$ \Comment{returns fully oriented graph}
\end{algorithmic}
\end{algorithm}

In this section, we introduce an algorithm 
that allows us to handle a more general class of cost functions, which includes both the linear and quadratic cost function that we just saw. We will assume all features are mutable, but make no assumptions about how many interventions may be induced, which may be infinite.
To recap, we saw in Algorithm~\ref{alg:quadratic_alg}
that discovering the graph requires access 
to a collection of interventional distributions, 
where each node $i$ is intervened on in at least one distribution of the set.
However, 
the per-node incentivization strategy will not always induce such a set.
And unlike Algorithm~\ref{alg:linear_alg}, we do not wish to rely on there being a finite number of inducible interventions.
Below, we will describe an adaptive exploration strategy
that ensures each node will be intervened upon, and develop a discovery algorithm to pair with this exploration strategy.



\paragraph{General Separable Heterogeneous Cost}~\label{cost_class:general} We now define a more general class of cost functions that contains both the linear and the quadratic cost function:

$$c(a; x) = \sum_{i=1}^n c_i(a_i; x_{S_i}),$$
where $S_i \subseteq [n]$ and the cost is heterogeneous when $S_i \neq \emptyset$. Each cost function $c_i$ is assumed to be strictly increasing in the magnitude of $a_i$ and again, we assume $c_i(0; x_S) = 0$. Note that we no longer require the condition $\frac{\partial c(a; x)}{\partial a_i} \Bigr|_{a_i=0} = c'_i(0; x_{S_i}) = 0$ as in Cost Function Class~\ref{cost_class:quadratic}.

\paragraph{Algorithm}  The key observation behind Algorithm~\ref{alg:general_alg}
is that if we know $g_i$, there is an apt choice of $f_i$ 
that is guaranteed to induce an intervention 
only on $X_i$. 
However, at the start, we do not know $g_i$, the SCM parameter. 
We develop an algorithm that learns $g_i$ on the fly, 
and uses it for orientation.

\begin{proposition}
\label{prop:only_intervention}
Deploying $X_i - g_i(X_{\pa(i)})$ induces 
an intervention only on node $X_i$.
\end{proposition}

At a high-level, Algorithm~\ref{alg:general_alg} is a bottom-up algorithm that leverages this fact to iteratively discover leaf nodes in the current subgraph. The crux of the algorithm is a leaf-node test based on the observation that intervening on a leaf only changes (and in particular increases) the leaf node itself and no other node in the subgraph. Identifying the leaf means that every node it is adjacent to in the subgraph must be a parent. We may then orient accordingly and recurse on the remaining subgraph. Please see Section~\ref{sec:gen_cost_proofs} in the Appendix for proof of correctness.

Our algorithm requires access to the natural distribution. 
And again, we will assume access to the graph skeleton, which has also been assumed to be accessible in prior works~\citep{miller2020strategic}.
Please see Appendix~\ref{sec:assumption_justification} for further discussion on these assumptions.

\begin{theorem}
Algorithm~\ref{alg:general_alg} recovers the full graph structure using at most $n(n-1) / 2$ deployments.
\end{theorem}

\begin{proofsketch}
Observe that if node $X_i$ is a leaf of the current subgraph $SG$ and everything in $S$ is downstream of $X_i$, then we can identify its true SCM parameters $g_i$ as $\hat{g}_i = (X_i | X_{P_i})_{\Dcal_0} = (X_i | X_{\pa(i)})_{\Dcal_0} = g_i$, since $P_i = \pa(i)$. Then, making use of Proposition~\ref{prop:only_intervention}, only $X_i$'s mean will shift in the current subgraph $SG$, when $f_i = X_i - \hat{g}_i(X_{P_i})$ is deployed. Notice also that since only $X_i$ is intervened upon, this minimally introduces confounding as there has to be at least one intervention on some node.

On the other hand, for non-leaf nodes, we prove that at least one node outside of $X_i$ in the subgraph will have its mean shift due to the induced intervention(s), requiring the mild assumption of Mean Interventional Faithfulness~\citep{zhang2021matching}.

With this established, we may iteratively discover a $X$ leaf node in the subgraph, orient its edges with parents and then recurse on the remaining subgraph. The only other case to handle is when no $X$ leaf node is found. By process of elimination, $Y$ must be the only leaf node in the current subgraph and its edges are oriented accordingly.

\end{proofsketch}

%% file: sections/50_discussion.tex
We highlight some conceptual insights
that arise from our framework and analysis:

\paragraph{Incentives as an Identification Tool}
In our work, we set out to explore learning 
Pareto optimal scoring mechanisms
in Causal Strategic Prediction via causal discovery.
To perform discovery, we develop algorithms
that leverage the different interventions induced
by different sets of incentives
to uncover the (initially unknown) causal graph. 
While \citet{miller2020strategic} link the 
hardness of discovering optimal mechanisms
to the hardness of causal discovery, 
they do not develop concrete causal discovery algorithms. 
We demonstrate that by observing 
the interventional distributions 
induced by a sequence of incentives,
we can exactly identify causal graphs 
that are otherwise only identifiable 
up to Markov equivalence. 
This initial work opens a new direction in causal discovery 
via soft interventions, which has traditionally
 relied on the ability to target variables directly, versus indirectly via incentives. 
%

\paragraph{Insights about Incentives and Causal Structure}
By considering general SCMs with arbitrary causal relationships
among the variables, 
we paint a richer picture capturing qualitative differences
between different types of variables (e.g., parents versus children)
that bear on their suitability for inclusion in the mechanism.
For example, 
prior works have cautioned against 
the inclusion of proxies in scoring mechanisms,
lest we fall victim to Goodhart's Law 
\citep{hardt2016strategic, karmo2014baltic}. 
However, our work reveals subtle insights
into the many considerations that influence
the suitability of causal descendants:
(i) how much we weigh predictive accuracy
versus improvement;
(ii) how correlated the proxy is with the outcome,
(iii) whether it is easier to manipulate the proxy directly
or indirectly, by intervening on its parents.

\paragraph{Beyond Sufficiency}
While our analysis focuses only on 
causally sufficient graphs, it is easy 
to look beyond the present setup
and see some immediate implications.
Consider the situation
where some parents of the outcome 
are not visible to the decision maker, 
and we possess a proxy 
that is expensive to manipulate.
Here incentives on proxies might induce 
interventions on the unobserved parents
and thus prove to be improvement-optimal, 
exceeding the improvement induced by incentivizing
the observable parents.
Proxies can also prove advantageous 
when sparsity of the mechanism is desirable~\citep{holmstrom1991multitask}.
Incentives on a single predictive but 
expensive-to-manipulate proxy 
can potentially induce interventions 
on a large number of causal parents.

\section{Future Work}\label{sec:future_work}

In closing, while our results establish 
connections between strategic responses and causality, we do make several idealizations to facilitate analysis. We would like to bring these assumptions and abstractions to the reader's attention, with the view that these are exciting directions to pursue in future work.

Firstly, as in prior causal, strategic ML works, we assume causal sufficiency and no sample complexity concerns. Removing the causal sufficiency assumption and accounting for sample complexity considerations are thus important future directions to address. As our results confirm the possibility of causal discovery under strategic manipulation, we believe a particularly exciting direction is developing \emph{practical} algorithms in this challenging setting.
Secondly, while the goal of the paper is to discover the causal graph, there is a need to consider what kind of scoring mechanisms are deployed during the discovery process and ensure they do not incur large risks \emph{during} the discovery process.
Finally, this paper relaxes the assumption that the principal knows the causal graph and assumes that only the individuals knows the true graph. A more general problem to study is when \emph{neither} the institution nor the individual knows the causal graph. We view our work as an useful stepping stone towards developing learning algorithms in this more general setting. The main challenge there is that now the principal needs to simultaneously discover the true graph, while accounting for the (possibly) misspecified causal graph of the individuals. Moreover, the individuals' perceived causal graph will also be updated during this interaction. And so, both moving parts need to be accounted for when performing causal discovery, which certainly adds to what is already a challenging task.

%% file: appendix.tex
\section{Gaming vs Improvement}\label{sec:game_vs_improve}

\begin{figure}[t]
\centering
\begin{subfigure}{.5\textwidth}
  \centering
  \includegraphics[width=0.7\linewidth]{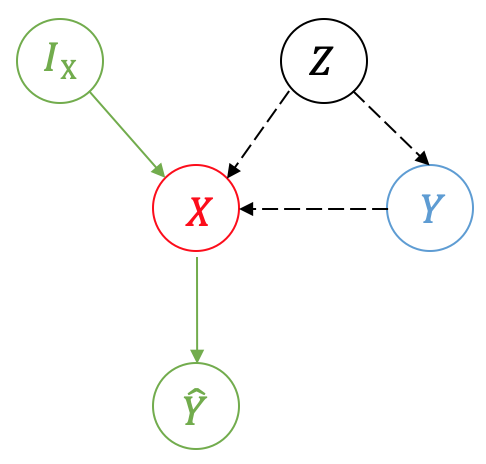}
\end{subfigure}%
\begin{subfigure}{.5\textwidth}
  \centering
  \includegraphics[width=0.6\linewidth]{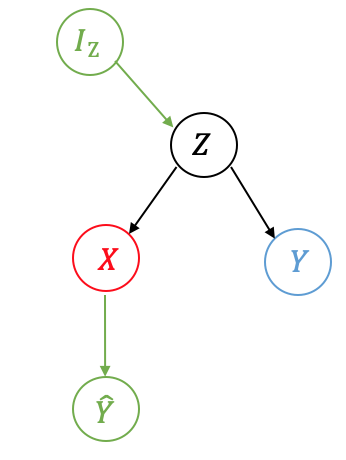}
\end{subfigure}%
\caption{Capturing graphical notion of gaming~\citep{hardt2016strategic} versus improvement~\citep{kleinberg2020classifiers}.}
\label{fig:improve_game_fig}
\end{figure}

Our framework encompasses two canonical viewpoints of feature changes in strategic machine learning: gaming and improvement.

\begin{enumerate}
    \item In the context of gaming under the~\citep{hardt2016strategic} setup, the goal is to predict $Y$ using a function of manipulated features $\hat{Y} = f(\tilde{X})$. The key underlying causal assumption behind gaming is that changes to $X$ will not affect $Y$. Hence, a necessary condition is that none of the features $X$ are an ancestor of $Y$.
    
    An example of such a setting is depicted in the left sub-Figure of Figure~\ref{fig:improve_game_fig}, in which $X$ may share a common parent with $Y$ and/or $X$ is a descendant of $Y$. Note that it has to be the case that if $X$ and $Y$ have common ancestors $Z$, $Z$ cannot be intervened upon. This may be captured in our setting by assigning infinite intervention costs on $Z$ to recover the assumptions made in the gaming setting.

    \item In the context of improvement under the~\citep{kleinberg2020classifiers} setup, the goal is to incentivize changes to some target $Y$ using a mechanism that is a function of $X$. More precisely, to use the nomenclature used in~\citep{kleinberg2020classifiers}, $X$ is the observed features, $Z$ is the (latent) ``effort profile''. The goal is to incentivize a best response profile $z$ that maximize some objective $g$, which we capture in $Y$ with $Y = g(Z)$ (note there is no exogeneous noise in this setting).
    
    To model improvement, a key assumption is that all interventions take place on $Z$ and not $X$, hence there is no gaming (due to inability to game). Again, this may be captured in our framework by setting the cost function to have infinite cost of intervening on $X$.
    
    It is assumed that the entire SCM is known (i.e the relationship between $X$ and $Z$ as well). Hence, the goal is to design $\hat{Y} = f(\tilde{X})$ so as to incentivize interventions $I_Z$ on $Z$, which maximally increases $Y$. Please see the right sub-Figure of Figure~\ref{fig:improve_game_fig} for a depiction of the graph.
    
\end{enumerate}

\newpage

\section{Regret as a Performance Measure}\label{sec:appendix_regret}

\subsection{Comparison with other Discovery Algorithms}

Another natural notion of algorithm optimality is regret. That is, let $\Pi$ be the set of all algorithms that can provably discovery the graph in the causal, strategic prediction setting. We may want to find low-regret discovery algorithm $\hat{\pi}$ such that the cumulative regret for some notion of loss $c$ is small:

$$\max_{G \in \Gcal} \left[\sum_{i=1}^{N^{\hat{\pi}}} c_G(\theta^{\hat{\pi}}_i) - \min_{\pi' \in \Pi} \sum_{j=1}^{N^{\pi'}} c_G(\theta^{\pi'}_j) \right]$$
where $\theta_i^{\pi}$ denotes the model deployed by algorithm $\pi$ at the $i$th step and $N^{\pi}$ denotes the total number of deployments needed by algorithm $\pi$ to provably discover graph $G$. 

Verily, while this may be the an apt goal to aim for, it is apriori unclear if discovery is even possible in this setting. Our paper is the first to address this question and establishes the first set of algorithms that can provably discover the graph in the causal strategic prediction setting. And so, we defer regret analysis to future work, once there is better characterization in the literature of the algorithms that are in $\Pi$.

\subsection{Comparison with Zeroth Order Optimization} 

Alternatively, since ultimately we care about optimizing for the tradeoff between improvement and risk, another natural algorithm that is applicable here is the set of Bayesian optimization algorithms that directly optimize for this loss (in face of unknown graph structure and SCM).

In this subsection, we compare our algorithm against the zeroth order Bayesian optimization procedure, which does not leverage causal structure. Our results demonstrate that the use of causal structure is useful for efficient optimization of the objective, where efficiency is measured in terms of the cost incurred during the optimization process.

\begin{figure}[t!]
    \centering
  \includegraphics[width=0.9\textwidth]{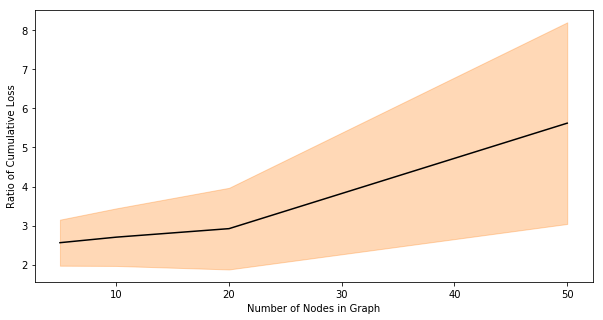}
    \caption{Comparison of cumulative negative improvement incurred during the GP Optimization process versus incurred during the per-node incentivization process. The error bars denote 95\% confidence intervals. The higher the ratio, the larger the cumulative loss incurred by the GP optimization \emph{relative} to the per-node algorithm.}
    \label{fig:improve_ratio_plot}
\end{figure}

\begin{figure}[t!]
    \centering
  \includegraphics[width=0.9\textwidth]{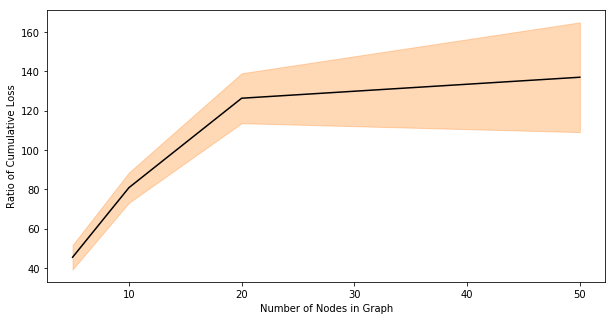}
    \caption{Comparison of cumulative risk incurred during the GP Optimization process versus incurred during the per-node incentivization process. The the error bars denote 95\% confidence intervals. The higher the ratio, the larger the cumulative loss incurred by the GP optimization \emph{relative} to the per-node algorithm.}
    \label{fig:risk_ratio_plot}
\end{figure}

We will compare our per-node algorithm against the Bayesian optimization procedure using Gaussian Processes. Note that our approach allows us to discover the graph, and use this to then optimize the tradeoff for any $\lambda > 0$. By constrast, the Bayesian Optimization procedure only allows for optimizing the tradeoff for some fixed $\lambda$, which determines the objective. Still, we will compare the two one one tradeoff parameter $\lambda$, set to be $1$ in the experiments:

\begin{equation}
\label{eqn:tradeoff}
\begin{aligned}
\min_{\theta} \Ecal_{(\tilde{X}, \tilde{Y}) \sim \Dcal(\theta)}[(f(\tilde{X}; \theta) - \tilde{Y})^2]  - \lambda \Ecal_{(\tilde{X}, \tilde{Y}) \sim \Dcal(\theta)}[g_Y(\tilde{X}_{\pa(Y)})]
\end{aligned}
\end{equation}

For brevity, let us write $R(\theta) = \Ecal_{(\tilde{X}, \tilde{Y}) \sim \Dcal(\theta)}[(f(\tilde{X}; \theta) - \tilde{Y})^2] $ and $I(\theta) = \Ecal_{(\tilde{X}, \tilde{Y}) \sim \Dcal(\theta)}[g_Y(\tilde{X}_{\pa(Y)})]$.

\paragraph{Experiment Description:} The setting we will work with is the linear SCM, quadratic cost setting~\citep{shavit2020causal}. 

We will be comparing the per-node Algorithm~\ref{alg:quadratic_alg} to the GP optimization algorithm. For the GP optimization algorithm, we will set a cutoff of $10^{-4}$ for optimization stoppage. 

We will evaluate both algorithms on a number of graphs by varying \verb|num_nodes|, the number of nodes in the graph. We will experiment with chain graphs, in which for each graph, the SCM parameters are sampled from $\text{unif}[-1,1]$ and the exogenous variables from $N(0,1)$. For a given \verb|num_nodes| value, we randomly generate $30$ graphs and average both algorithm's performance across these graphs.

To measure the performance of both algorithms, we will compare the cumulative loss incurred by the two algorithms by computing the ratio between the GP optimization algorithm and the per-node algorithm. The higher the ratio, the larger the cumulative loss incurred by the GP optimization \emph{relative} to the per-node algorithm.

The per-step loss $c$ that makes up this cumulative loss will be taken as either the MSE of the chosen $\theta$, $R(\theta)$, or the negative improvement incurred by $\theta$, $\ind\{I(\theta) < 0\} I(\theta)$. We estimate these two expectations by drawing $10000$ samples from the distribution and then evaluating the loss accordingly.

As can be seen in Figure~\ref{fig:improve_ratio_plot} and Figure~\ref{fig:risk_ratio_plot}, in terms of both improvement $I(\theta)$ and risk $R(\theta)$, the cumulative loss incurred by the per-node algorithm is lower than and a small fraction of the cumulative loss incurred by the zeroth order optimization procedure. Note that the ratio is much larger when comparing risks instead of improvement. This may be because we only sum over only the improvement if it is negative improvement. By contrast, for risk, we sum over the risk incurred during every time step.

Overall, this experiment suggests that, in terms of cumulative loss, discovering the causal structure is more efficient than optimizing the unknown objective without learning and then leveraging this causal structure.

\newpage

\section{Deferred Material from Section~\ref{sec:tradeoff}}\label{sec:tradeoff_proofs}

\subsection{The Improvement vs Risk Tradeoff}

\begin{proposition}
Model $f = g_Y(X_{\pa(Y)})$ maximizes improvement and minimizes risk.
\end{proposition}
\begin{proof}

\textbf{$g_Y(X_{\pa(Y)})$ is Improvement-Optimal:} Writing out the program, we have that the optimization objective for an individual with features $x$ is:
\begin{equation}
\begin{aligned}
\max_{a} \quad & g_Y(x'_{\pa(Y)})\\
\textrm{s.t.} \quad & x'_i = g_i(x'_{\pa(i)}, u_i) + a_i \quad \forall i \in [n]   \\
\quad & c(a_1, ..., a_n; x) \leq b
\end{aligned}
\end{equation}

Note that for each individual, $u_Y$ is realized and a fixed value. Thus, using that the SCM is additive, $y' = g_Y(x'_{\pa(Y)}) + u_Y$. This optimization program is equivalent to:

\begin{equation}
\begin{aligned}
\max_{a} \quad & y' \\
\textrm{s.t.} \quad & x'_i = g_i(x'_{\pa(i)}, u_i) + a_i \quad \forall i \in [n]   \\
\quad & c(a_1, ..., a_n; x) \leq b
\end{aligned}
\end{equation}

where $y'$ is the value of $y$ under soft intervention $a$ on the causal graph. 

Thus, we have that this scoring mechanism incentivizes the intervention $a^*$ (out of all feasible interventions) that maximizes $\tilde{y}$ for every individual. Hence, this scoring mechanism is improvement optimal.

\textbf{$g_Y(X_{\pa(Y)})$ is Risk-Optimal:} We first prove that the MSE of any mechanism $f$ is lower-bounded by $\var(U_Y)$. WLOG let $\Ecal[U_Y] = 0$. Then:
\begin{align*}
\Ecal[(f(\tilde{X}; \theta) - \tilde{Y})^2] & = \Ecal[(f(\tilde{X}; \theta) - g_Y(\tilde{X}_{\pa(Y)}) - U_Y)^2] \\
& = \Ecal[(f(\tilde{X}; \theta) - g_Y(\tilde{X}_{\pa(Y)}))^2] + \Ecal[U_Y^2] \\
& \geq \var(U_Y),
\end{align*}
where the second equality follows from that $U_Y \indep \tilde{X}_{\cdot}$, since every node $\tilde{X}_{\cdot}$ is an ancestor of $Y$.

The result follows from the observation that using $f = g_Y(X_{\pa(Y)})$ attains a MSE of $\var(U_Y)$.

\end{proof}

\begin{proposition}
In Example~\ref{example: two_node}, there exists a SCM and cost structure, where the optimal improvement is $1$ and:

\begin{enumerate}
    \item Any mechanism with low risk $o(\epsilon)$ must have improvement at most $O(\epsilon)$
    
    \item There is a mechanism $f$, which is a function of only the proxy, that has low risk $o(\epsilon)$ and also a high improvement of at least $1 - \epsilon$.
\end{enumerate}

\end{proposition}

\begin{proof}

Consider the setting where the cost is quadratic and equal across the two features, $c(a) = a_1^2 + a_2^2$, and budget $b=1$.

Define $v_1 = \var(X_1), v_Y = \var(U_Y)$ and $v_2 = \var(U_2)$. Then, for linear model $f = w_1X_1 + w_2X_2$, 
its MSE in terms of the variances has closed-form: $(w_1 + \alpha_2w_2 - 1)^2v_1 + (w_2\alpha_2 - 1)^2 v_Y + w_2^2v_2$. 

The optimal best response has $(a^*_1, a^*_2) = (\frac{w_1 + \alpha_2w_2}{\sqrt{(w_1 + \alpha_2w_2)^2 + w_2^2}}, \frac{w_2}{\sqrt{(w_1 + \alpha_2w_2)^2 + w_2^2}})$. And so, the improvement induced by $f$ has closed form: 
$\frac{w_1 + \alpha_2w_2}{\sqrt{(w_1 + \alpha_2w_2)^2 + w_2^2}}$. The optimal improvement is $1$ and is attainable with $w = (1, 0)$ (only the parent is incentivized).

Let SCM be such that $v_1 = v_Y = 1 / \epsilon$ and $v_2 = \epsilon^4$. This means that the variance of $U_Y$ is large, which makes $X_1$ a noisy predictor of $Y$. By contrast, the variance of proxy $X_2$ is small, which makes $X_2$ a good predictor of $Y$:

\begin{itemize}
    \item First consider a SCM where $\alpha_2 = \epsilon$ is small.
    
    We know that the optimal risk attainable is upper bounded by $\epsilon^2$, which is the risk of $w = (0, 1 / \alpha_2) = (0, 1 / \epsilon)$. So it is possible to have very predictive models.
    
    Now, for any predictive model with near-optimal risk $o(\epsilon)$,
    we need that $|w_2\alpha_2 - 1| = o(\epsilon)$ 
    and $|w_1 + \alpha_2w_2 - 1| = o(\epsilon)$. 
    The former implies that $w_2 = 1 / \epsilon + o(1)$ 
    and in combination with the second condition implies that $w_1 = o(\epsilon)$. 
    
    This means that $w_1 + \alpha_2w_2 = 1 + o(\epsilon)$. And so, the improvement of any model that has low-risk can be at most $O(\epsilon)$.

    \item Now, consider a SCM where $\alpha_2 = 1 / \epsilon$ is large.

    We claim that $f = 1 / \alpha_2 X_2$ will now have both low risk and high improvement.

    Indeed, under $w = (0, 1 / \alpha_2) = (0, \epsilon)$, the risk is $\epsilon^2 \cdot v_2 = \epsilon^6$. 
    
    Moreover, $w$ attains a near-optimal improvement: $\frac{w_1 + \alpha_2w_2}{\sqrt{(w_1 + \alpha_2w_2)^2 + w_2^2}} = \frac{1}{\sqrt{1 + \epsilon^2}} \geq 1 - \epsilon$.

    The key difference here is that while we again have to resort to using proxy $X_2$ to predict $Y$, since $\alpha_2$ is large, under equal cost, almost all the budget will be invested in $X_1$ instead of $X_2$. Because increasing $X_1$ increases $X_2$ much more than increasing $X_2$ itself, $X_2$ also a good incentive.

\end{itemize}

\end{proof}

One may naturally conjecture that $\Ecal[Y | X_{\pa(Y)}] = g_Y(X_{\pa(Y)})$ may in fact be improvement-optimal across all graphs. In Proposition~\ref{prop:improve_general_graph} below, we show that this is not the case. It turns out that proxies are important for not just for predicting $Y$, but also for designating incentives in general graphs.

\begin{proposition}\label{prop:improve_general_graph}
$f = g_Y(X_{\pa(Y)})$ is not improvement-optimal under general graphs.
\end{proposition}
\begin{proof}

Consider the setting where the cost is quadratic and equal across the two features, $c(a) = a_1^2 + a_2^2$, and budget $b=1$.

Now, suppose the causal graph is as follows: $X_1 \rightarrow Y \rightarrow X_2$. And the SCM is such that $Y = U_Y X_1$, $X_2 = Y$, where $X_1 \in U[2, 3]$ and $U_Y = 1$ with probability $p$ and $U_Y = -1$  otherwise. Let $p > 1 / 2$.

$\Ecal[Y | X_{\pa(Y)}] = \Ecal[U_Y]X_1 = (2p - 1)X_1$. This will induce an intervention of $a^*_1 = 1$ for all and induce improvement $p + -(1 - p) = 2p - 1$.

Across all possible models, the optimal improvement is in fact $1$. One improvement optimal model is: 
$$f = \ind\{X_2 < 0\}(-X_1) + \ind\{X_2 \geq 0\}(X_1).$$

To see this, first note that the best response will be such that $a_2^* = 0$. For any individual, since $|x_2| \geq 2$, any change to $x_2$ will not flip the sign and thus change the objective.

Now, since $X_2 < 0 \Leftrightarrow U_Y < 0 \Leftrightarrow g_Y = -X_1$, for an individual with $U_Y = -1$, we incentivize $f = -X_1$, which leads to $a^* = -1$. And similarly, since $f$ is monotonic in $a_1$, for $U_Y = 1$, we incentivize $f = X_1$, which leads to $a^* = 1$.

The improvement induced by the mechanism is thus $p + (1 - p) = 1$. This is optimal across all models, since the improvement induced by any model is upper-bounded by $1$: for any individual, $\tilde{Y} = U_Y \tilde{X}_1 = U_Y(X_1 + a^*_1) = Y + U_Ya^*_1 \leq Y + 1$.

\end{proof}

\textbf{Remark:} Note that $\Ecal[Y | X_{\pa(Y)}]$ will actually induce the worst possible improvement for the sub-population with $U_Y = -1$. And so, $X_2$ is actually very important in terms of designing a more targeted incentive.

\subsection{SCM Recovery for Identifying the Pareto Front}\label{subsec:scm_identifiability}

\textbf{Identifying the SCM:} Once we know the graph structure, we note that an additive SCM may be efficiently identified from the natural distribution $\Dcal_0$ as follows in Algorithm~\ref{alg:scm_id}:

\begin{algorithm}[h]
\caption{Subroutine for identifying $g_V$ for each node $V$}
\label{alg:scm_id}
\begin{algorithmic}[1]
\State Input: Distribution $\Dcal_0$, Graph Structure $\Gcal$
\For{each node $V \in \{X_j\}_{j \in [n]} \cup \{Y\}$}
    \State Set $g_V = \Ecal_{\Dcal_0} [V | \pa_{\Gcal}(V)]$
\EndFor
\end{algorithmic}
\end{algorithm}

In particular, due to L3 (functional) equivalence, we may without loss of generality assume that $\Ecal[U_i] = 0$, and the ANM SCM parameter may be identified up to constant shifts.

Then, for any node $V$, the ANM SCM parameter is identified from the following regression:
$$\Ecal_{\Dcal_0}[V | \pa_{\Gcal}(V)] = \Ecal_{\Dcal_0}[g_V(\pa_{\Gcal}(V)) + U_V | \pa_{\Gcal}(V)] = g_V(\pa_{\Gcal}(V)) + \Ecal_{\Dcal_0}[U_V] =g_V(\pa_{\Gcal}(V)).$$

\textbf{Recovery of distribution of $U_V$:} the distributions of the exogenous variables are identified once the SCM is identified. For each node $V$, we can obtain $U_V = V - g_V(\pa_{\Gcal}(V))$ from natural distribution $\Dcal_0$.

\textbf{Identifying the Pareto front:} the Pareto front is identified if, for any $\theta$, the distribution $\Dcal(\theta)$ induced by $f(\cdot; \theta)$ is identified. 

$\Dcal(\theta)$ may be generated by computing agents' resultant features by first drawing realized, exogenous variables $(u_1, ..., u_n)$, computing their unperturbed features $x$ from $(u_1,...,u_n)$ (that determines the cost) and then computing their best response under $f(\cdot; \theta)$:

\begin{equation}
\label{eqn:best_response}
\begin{aligned}
a^* = \arg\max_{a} \quad & f(x'; \theta)\\
\textrm{s.t.} \quad & x'_j = g_j(x'_{\pa(j)}) + u_j +  a_j \quad \forall j \in [n]   \\
\quad & c(a_1, ..., a_n; x) \leq b.
\end{aligned}
\end{equation}
Note that the corresponding, resultant feature $\tilde{x}$ is such that coordinate-wise: $\tilde{x}_j = g_j(\tilde{x}_{\pa(j)}) + u_j +  a^*_j$.


\newpage

\section{Proof Preliminaries}

\paragraph{Notation} We summarize some conventions used in the paper. Upper case letters to denote random variables $X$ and lower case to denote the realized value of a random variable e.g $X = x$.

For brevity, we use $\Dcal_j$ to refer to the distribution induced by mechanism $f = X_j$. $\tilde{X_i}$ is used to denote the feature $X_i$ in the distribution induced by the deployed mechanism, which will be be specified within the context that $\tilde{X}_i$ is referenced.

Also, we will interchangeably refer $\tilde{X}_j$ as $\tilde{X}_j$  or node $j$, and likewise $X_j$'s ancestors  as either $\anc(X_j)$ or $\anc(j)$. Let the set of feature nodes $\{X_i\}_{i=1}^n$ be $\Xcal$. We will use $(X_{S_1} | X_{S_2})_{\Dcal_i}$ as another way to denote the conditional distribution of $\tilde{X}_{S_1} | \tilde{X}_{S_2}$ under $\Dcal_i$.

\paragraph{Graph Properties:} To analyze how nodes are confounded by interventions, we will make heavy use of graphical criteria. As depicted in Figure~\ref{fig:confounding_fig}, the causal graph $G$ consists of nodes $X_1, ..., X_n$ and $Y$. Let graph $\tilde{G}$ comprise of nodes $\tilde{X}_1, ..., \tilde{X}_n$ and $\tilde{Y}$, which correspond to the features observed in the induced distribution. 

$\tilde{G}$'s structure and SCM will mirror that of $G$. Different from $G$, $\tilde{G}$ will have intervention nodes pointing into node $\tilde{X}_j$ iff $\tilde{X}_j$ is intervened upon. Under intervened distribution $\Dcal_i$, we will call the augmented intervention node pointing into node $\tilde{X}_j$, $I^i_j$.

\section{Deferred Material on Algorithm~\ref{alg:quadratic_alg}} \label{sec:quad_cost_proofs}

\subsection{Algorithm Assumptions}

\textbf{Cost Function Property:} To recap, the key property of Cost Function Class~\ref{cost_class:quadratic} that we will make use of  is that: 
for each $i \in [n]$ and 
for all attainable values of $\{x_j\}_{j \in {\anc(i)  \cup \{i\}}}$, $\frac{\partial c(a; x)}{\partial a_i} \Bigr|_{a_i=0} = c'_i(0; x_{S_i}) = 0$. 
That is, the marginal cost of changing $a_i$ when $a_i = 0$ is always zero.

Note that this condition rules out cost functions where for a descendant $X_j$ of $X_i$, one may choose to set $a_j \neq 0$ in order to lower the cost of changing $a_i$; allowing for this will induce arbitrarily complex confounding. Moreover, the latter condition rules out cost functions like linear cost, which we address later through a different algorithm.

\textbf{Applicable SCM:} The following discovery algorithm is applicable to any causal graph.

\subsection{Proof of Theorem~\ref{thm:quad_thm}}

\subsubsection{Characterization of Confounding}

To begin, we wish to characterize the confounded graph $\tilde{G}$ under per-node incentivization and Cost Function Class~\ref{cost_class:quadratic}. This characterization requires knowing which nodes are intervened upon and as a function of which variables. 

When the objective is $f = x_i$, the individual is optimizing:

\begin{equation}
\label{eqn:alg1_br}
\begin{aligned}
\max_{a} \quad & x'_i \\
\textrm{s.t.} \quad & x'_j = g_j(x'_{\pa(j)}, u_j) + a_j \quad \forall j \in [n]   \\
\quad & \sum_{j = 1}^n c_j(a_j; x_{S_j}) \leq b.
\end{aligned}
\end{equation}

Recursively unrolling the objective using the SCM constraints (starting with $x'_i = g_i(x'_{\pa(i)}, u_i) + a_i$), we observe that the objective is a function of variables: $u_{\anc(i) \cup \{i\}}$ (which are fixed) and $a_{\anc(i) \cup \{i\}}$.

Then, using that $c_j$ is strictly increasing in the magnitude of $a_j$ for all $j \in [n]$, we have that $a^*_{k} = 0$ for $k \not \in \anc(i) \cup \{i\}$. Indeed, if some $a^*_{k} = \Delta \neq 0$, then setting $a^*_{k} = 0$ and increasing $a_i$ by a nonzero amount (subject to budget constraint and $\Delta$) will strictly increase the objective. It is always possible to increase $a_i$ by a nonzero amount, since $c_i$ is finite everywhere. In summary, we have that for $a^*_k \neq 0$ only if $k \in \anc(i) \cup \{i\}$, and its value is a function of some subset of $x_{\anc(i) \cup \{i\}}$. 

\textbf{Key Property:} Moreover, the key property we will make use of for this class of cost functions is that under deploy $f = x_i$, in $\Dcal_i$, node $i$ is guaranteed to be intervened upon.

To see this, since $\nabla c \neq 0$ for $a$ on the surface $c(a;x) = b$, we may characterize all stationary points, which includes the best response, via Lagrange multipliers as follows: $a^*$ is such that $1 = \frac{\partial f}{\partial a_i} \Bigr|_{a = a^*} = \lambda \frac{\partial c(a; x)}{\partial a_i} \Bigr|_{a = a^*}  = \lambda c_i'(a^*_i; x_{S_i})$, for $\lambda$ the Lagrange Multiplier. 
With this, we can conclude that $a^*_i \neq 0$, since if $a^*_i = 0$, then the LHS is $1$, while the RHS is $0$.

\subsubsection{Main Proof}

\textbf{Faithfulness Assumptions:} In order for discovery to succeed, some notion of faithfulness is required. For our algorithm to complete, we will require the following faithfulness assumption, which is discussed further in subsection~\ref{subsec:alg_1_faithfulness}. Intuitively, the condition ensures that if a conditional distribution $V | S$ is dependent on some intervention in $\Dcal_i$, then this intervention causes the conditional distribution to shift and differ from the natural conditional distribution $V | S$ under $\Dcal_0$. Conversely, if this conditional distribution is independent of every intervention, then it is no different from the natural conditional distribution.

\textbf{Assumption (Parameter Faithfulness):} \label{assumption:quad_alg_assump} Let $V \in G$ be any node in $G$ and $S$ any subset $S \subset G$ with $V \not \in S$. Let $\Ical^i$ be the set of all non-$\verb|null|$, intervention nodes in the augmented graph $\tilde{G}$ corresponding to $\Dcal_i$, then:
$$\exists I \in \Ical^i \text{ s.t } I \notindep  \tilde{V} | \tilde{S} \Leftrightarrow  (V | S)_{\Dcal_i} \neq (V | S)_{\Dcal_0}$$

\begin{theorem}
Algorithm~\ref{alg:quadratic_alg} recovers the full graph structure with $n$ deployments.
\end{theorem}
\begin{proof}

To prove algorithm correctness, we will prove that the children of every node are correctly identified by the algorithm, which implies that the full graph is correctly identified. To do this, we will show that the algorithm always (1) adds a root node of the subgraph $SG$ to the set of nodes $S$.

If (1) is satisfied, then it implies that (2) no node is added before all of its ancestors. This is because when a node is added to $S$, it must be a root and a node with at least one ancestor in the current subgraph cannot be a root node. 

Thus, with (1), the algorithm will be such that the following holds: (3) each node in $S$ has its children correctly and completely identified. This is because when a root node of the subgraph is added to $S$, we identify all nodes adjacent to it in the $SG$ as its children. All nodes adjacent to the root node in $SG$ must only be its children: if a node is adjacent to a parent in $SG$, it will not be a root. Moreover, the nodes adjacent to the root must be all of the children, since by (2) all its children must still be in $SG$. Thus, the algorithm's orientation ensures that the newly added node also satisfies property (3).

We see that (1) is satisfied for $S$ at initialization. To prove (1) always holds, it suffices to show that only root nodes in any subgraph $SG$ will be such that Condition~\ref{alg_line:natural_root_test} is always false and no non-root nodes in any subgraph $SG$ can have Condition~\ref{alg_line:natural_root_test} be always false. For a particular subgraph $SG$, suppose $SG$ and $S$ satisfies (1).

\textbf{Non-Root nodes in $\Xcal$ do not pass test:} We will first show that for any non-root node, Condition~\ref{alg_line:natural_root_test} will hold at least once. This is true because if $\tilde{X}_i$ is a non-root node, it must have a parent $V$.

Consider $(V | X_i, \anc_{GS}(X_i))_{\Dcal_i} = (V | X_i, \anc_{GS}(X_i))_{\Dcal_0}$. In $\tilde{G}$, we have that $V \rightarrow \tilde{X}_i \leftarrow I_i^i$. Thus, $V$ is dependent on $I_i^i$ due to collider $\tilde{X}_i$, which is conditioned upon. Note that from our characterization, since $a^*_i \neq 0$, we have that $I_i^i \neq \verb|null|$. Therefore, from the Parameter Faithfulness condition, $I_i^i \notindep \tilde{V} | \tilde{X}_i, \tilde{X}_{\anc_{GS}(X_i)} \Rightarrow (V | X_i, \anc_{GS}(X_i))_{\Dcal_i} \neq (V | X_i, \anc_{GS}(X_i))_{\Dcal_0}$.

\textbf{Root Nodes in $\Xcal$ pass test:} Let one root node in subgraph $SG$ be $X_i$. Then, by (2), all its ancestors must be in $S = \Xcal \setminus SG$. Since by (3), the children of each node in $S$ are correctly identified, this means that there is an edge in the current $GS$ between $X_i$ and all of its parents, which are in $S$. Moreover, the same holds for $X_i$'s parents and the parents of its parents, as they too are in $S$. Inductively, this means that all of $X_i$'s ancestors are fully identified in $GS$. That is, $\anc_{GS}(X_i)$, which is the set of ancestors of $X_i$ under $GS$ (the current partially oriented DAG) is the true set of ancestors of $X_i$: $\anc_{GS}(X_i) = \anc(X_i)$.

Since $X_i$ is the root in $SG$, we have that every node adjacent to $X_i$ in the $GS$ must be its child and by (2), the set of nodes adjacent to $X_i$ in $SG$ must be its complete set of children.

Let $V$ be a node adjacent to $X_i$, i.e $V$ is a child of $X_i$. Note that $I_k^i \neq  \verb|null|$ only if $k \in \{i\} \cup \anc(X_i)$ from our characterization of confounding (that interventions only happens on $i$ and some subset of $\anc(i)$). So it suffices to check independence of $V$ with respect to each node in the set $\{I_{k}^i: k \in \{i\} \cup \anc(X_i)\}$. 

This we show in Lemma~\ref{lemma:blocked_path}: for all $k \in \anc(i) \cup \{i\}$, $\tilde{V} \perp I_{k}^i | \tilde{X}_i, \tilde{X}_{\anc(X_i)}$. And so, since $V$ is independent of all non-null intervention nodes conditioned on $\tilde{X}_i, \tilde{X}_{\anc_{GS}(X_i)}$, we have from the Parameter Faithfulness assumption that:
$$\forall k \in \anc(i) \cup \{i\}, \tilde{V} \perp I_{k}^i | \tilde{X}_i, \tilde{X}_{\anc_{GS}(X_i)} \Rightarrow (V | X_i, \anc_{GS}(X_i))_{\Dcal_i} = (V | X_i, \anc_{GS}(X_i))_{\Dcal_0}.$$

In summary, Condition~\ref{alg_line:natural_root_test} will never hold if $X_i$ is a root. And so, if there is at least one root in $SG$ that is not $Y$, one of these roots will be added to $S$ as a root of the subgraph. Thus, (1) will be preserved after the new addition, since a node from $SG$ is added iff it is a root.

\textbf{Lone-$Y$ root:} The only other case to consider is when $Y$ is the only root in the subgraph $SG$. As we have shown, if there is a non-$Y$ root in $SG$, it will meet the criteria and be added to $S$. Moreover, we have shown that no non-root node in $\Xcal$ will be such that Condition~\ref{alg_line:natural_root_test} is always false. Hence, if no nodes in $\Xcal$ meet the criteria, it must be the case that, by process of elimination, $Y$ is the only root of the subgraph and will be added accordingly.

\textbf{Termination:} The algorithm terminates when there is only one node left in the subgraph. By (2), it must be a leaf in the full graph. This means that we have also completely identified its children, which is the empty set.

\textbf{Complexity:} Each iteration we run $O(n^2)$ regressions (upper bounded by twice the number of edges in the subgraph) and there are $n$ iterations. $n$ deployments are needed to generate $\{\Dcal_i\}_{i=1}^n$ used to identify the graph.

Note that enumerating all nodes adjacent to a node $X_i$ is needed in Condition~\ref{alg_line:natural_root_test} since there exists graphs like the ``star'' graph, where the center node of the star is not a root, but has only one parent. Such a graph would lead to Condition~\ref{alg_line:natural_root_test} being false for all but one of the nodes (i.e the one parent).

\end{proof}

\textbf{Remark:} Note that our algorithm only makes use of comparisons between the interventional distribution and the natural distribution. Even faster discovery may be possible if we are to use logical rules such as Meek's rules on top of the discovery step (Line~\ref{alg_line:quad_orientation}).

\subsection{Main Lemma}

A key lemma needed in the main proof is the following. Suppose node $V$ is a child of $X_i$, then for all $k \in \{i\} \cup \anc(i)$, $\tilde{V} \perp I_{k}^i | \tilde{X}_i, \tilde{X}_{\anc(X_i)}$. To show this statement, we will prove the following:

\begin{lemma}~\label{lemma:blocked_path}
For any $k \in \{i\} \cup \anc(i)$, any path from $I_k^i \squigglypath \tilde{V}$ is blocked by nodes $\tilde{X}_i, \tilde{X}_{\anc(X_i)}$.
\end{lemma}

\begin{proof}

Suppose by contradiction, there is an unblocked path between $I_k^i$ and $\tilde{V}$. Without loss of generality, let this path contain only distinct nodes. We will prove the statement in three parts:

\begin{enumerate}
    \item First, we will argue that this path cannot contain any nodes are in $G$. 

    Since removing nodes $\{U_m\}_{m=1}^n$ disconnects  $X_1, ..., X_n, Y$ (graph $G$) from $\tilde{X_1}, ..., \tilde{X_n}, \tilde{Y}$ (graph $\tilde{G}$), if the path $I_k^i \squigglypath \tilde{V}$ contains nodes in $G$, it must cross from $\{U_m\}_{m=1}^n$ into $G$ and then cross back from $G$ into $\{U_m\}_{m=1}^n$ at least once.
    
    Thus, there exists a segment starting at $U_l$ and ending at $U_r$ for some $l$ and $r$, such that every node in the segment $U_{l} \squigglypath . U_{r}$  lies in $G$.
    
    In this case, we know then that the penultimate node must then be $X_r$, since the only node in $G$ that $U_r$ is adjacent to is $X_r$ and similarly for $U_l$: the segment is thus of the form $U_l \rightarrow X_l \squigglypath . X_r \leftarrow U_r$.
    
    Furthermore, the node $X_{r'}$ that $X_r$ is adjacent to must form a chain, as otherwise $X_{r'} \rightarrow X_r \leftarrow U_r$ is blocked since $X_r$ (and also its descendants which are in $G$) is not conditioned upon. Note that if a collider's descendant is conditioned upon, then so is the collider.
    
    Inductively, we must have that the segment from $X_r$ to $X_l$ must be a chain, i.e $X_l \squigglypath .... \leftarrow X_{r'} \leftarrow X_r \leftarrow U_r$. This follows the same reasoning that if not, there would exist a collider in $G$ that is not conditioned upon.
    
    However, the segment would then be of the form $U_l \rightarrow X_l \leftarrow ... \leftarrow X_{r'} \leftarrow X_r \leftarrow U_r$, making $X_l$ a collider that blocks the path. 
    
    In summary, $l$ has to be downstream of $r$ in order for the path within $G$ from $X_{r}$ to $X_l$ to be unblocked. But then, we would already have a collider in $X_{l}$.

    \item It remains to consider paths that only contain nodes from $\{U_m\}_{m=1}^n$, $\{I_{m}^i\}_{m=1}^n$ and $\tilde{G}$.

    We will make the following observation: 
    
    Any path from a node in $\{U_m: m \in \{i\} \cup \anc(i)\} \cup \{I_m^i: m \in \{i\} \cup \anc(i)\}$ to $\tilde{V}$ must contain some $\tilde{X}_{k'}$ for $k' \in \{i\} \cup \anc(i)$.
    
    This follows because each $U_m$ and $I_m^i$ is adjacent to only $\tilde{X}_m$ in $\tilde{G}$. And so, removing all of $\{\tilde{X}_m: m \in \{i\} \cup \anc(i)\}$ would disconnect $\{U_m: m \in \{i\} \cup \anc(i)\} \cup \{I_m^i: m \in \{i\} \cup \anc(i)\}$ from the rest of $\tilde{G}$, and in particular from node $\tilde{V}$ in $\tilde{G}$.
    
    \item Having established this, consider the ancestor of $\tilde{X}_i$ closest to $\tilde{V}$ on this path. This node must exist due to the preceding point. Let this node be $\tilde{X}_{k'}$, and let the two nodes adjacent to $\tilde{X}_{k'}$ on this path be $V_1$ and $V_2$ with $V_2$ the node closer to $\tilde{V}$.
    
    Then, $\tilde{X}_{k'}$ must form a collider $V_1 \rightarrow \tilde{X}_{k'} \leftarrow V_2$. This is because $\tilde{X}_{k'}$ is conditioned upon, and would otherwise block in a chain $V_1 \rightarrow \tilde{X}_{k'} \rightarrow V_2$, $V_1 \leftarrow \tilde{X}_{k'} \leftarrow V_2$ or as a common parent $V_1 \leftarrow \tilde{X}_{k'} \rightarrow V_2$.

    Now, $V_2$ cannot be in $\tilde{G}$, as otherwise $V_2$ is an ancestor of $\tilde{X}_i$ and is closer to $\tilde{V}$ than $\tilde{X}_{k'}$. So $V_2$ must be either $U_{k'}$ or $I_{k'}^i$ since these two nodes are the only other nodes $\tilde{X}_{k'}$ is adjacent to. However, from the previous point, the path from $V_2$ to $\tilde{V}$ must then go through another ancestor of $\tilde{X}_i$, which contradicts the minimality of $\tilde{X}_{k'}$.
    
\end{enumerate}
\end{proof}

\subsection{Faithfulness Assumption under Linear SCM and Quadratic Cost}\label{subsec:alg_1_faithfulness}

In this section, we explore the condition implied by Assumption~\ref{assumption:quad_alg_assump} under Linear SCM and Quadratic Cost, which is the canonical setting where the best response is smooth in $\theta$. And unlike linear cost, there is an uncountable number of possible interventions. Note that the faithfulness assumption assumed is an ``indirect'' variant of the direct $\mathcal{I}-$faithfulness assumption, which has been shown to be required for the underlying $\mathcal{I}$-Markov equivalence class to be identifiable (i.e necessary in the setting where the intervention targets are unknown). 

\subsubsection{Quadratic Best Response} 

In this setting, each individual is solving:
\begin{equation}
\begin{aligned}
\max_{a} \quad & w^T(x + Ba)\\
\textrm{s.t.} \quad & \frac{1}{2} a^T Ca \leq b.   \\
\end{aligned}
\end{equation}
We assume $C$ is diagonal. Due to scale being a degree of freedom, without loss of generality, $b=1$. With $\lambda$ the Lagrange multiplier, we have that:
$$B^Tw = \lambda C a \Rightarrow a^* = \frac{1}{\lambda} C^{-1} B^Tw,$$
and $\lambda = \sqrt{\frac{1}{2b} w^TBC^{-1}B^Tw }$ to satisfy the feasibility constraint.

\subsubsection{Implications of the Faithfulness Assumption} 

For linear SCMs, the corresponding condition is as follows. 

\begin{proposition}
Let $\mathbf{M} = (X_i, \anc'(X_i))$ represent the regressors, $\Sigma = \text{Var}(\mathbf{M})$ be the covariance matrix of the regressors, and $\sigma_{\textbf{M}, V} \in \mathbf{R}^{|\mathbf{M}|}$ be the vector denoting the covariance of $\textbf{M}$ and $V$, i.e., $(\sigma_{\textbf{M}, V})_j = \text{Cov}(\textbf{M}_j, V)$. Let $\Delta^{(i)}$ encode the shifts due to agent responses in distribution $\Dcal_i$. Then parameter faithfulness holds iff $\Delta^{(i)}_V \neq (\Delta^{(i)}_{\textbf{M}})^\top \Sigma^{-1} \sigma_{\textbf{M}, V}$.
\end{proposition}

\begin{proof}

Let's assume that under the natural distribution $\Dcal_0$, the data $X_{\Dcal_0}$ satisfies
\begin{align*}
    X_{\Dcal_0} = X + \mu^{\Dcal_0}, \,\, \text{where} \\
    X = B X + U, \,\, \text{with} \,\, \mathbf{E}[X] = 0,
\end{align*}
and $\mu^{\Dcal_0}$ is a fixed vector that is the mean of $\Ecal_{\Dcal_0}[X]$. 

The regression coefficients (in the population limit) from the regression of $V$ on $\textbf{M}$ (including the intercept term) in block matrix form is
\begin{align*}
    \beta^{\Dcal_0}_{V | \textbf{M}} &= \mathbf{E}\left[ [1, \mathbf{M}] [1, \mathbf{M}]^\top \right]^{-1} \mathbf{E}\left[ [1, \mathbf{M}] V \right] \\
    &= \begin{bmatrix}
    1 & (\mu^{\Dcal_0}_{\mathbf{M}})^\top \\
    \mu^{\Dcal_0}_{\mathbf{M}} & \Sigma + \mu^{\Dcal_0}_{\mathbf{M}} (\mu^{\Dcal_0}_{\mathbf{M}})^\top
    \end{bmatrix}^{-1} \begin{bmatrix}
    \mu^{\Dcal_0}_V \\
    \sigma_{\textbf{M}, V} + \mu^{\Dcal_0}_V \mu^{\Dcal_0}_{\mathbf{M}}
    \end{bmatrix} \\
    &= \begin{bmatrix}
    1 + (\mu^{\Dcal_0}_{\mathbf{M}})^\top \Sigma^{-1} \mu^{\Dcal_0}_{\mathbf{M}}  & -(\mu^{\Dcal_0}_{\mathbf{M}})^\top \Sigma^{-1} \\
    \Sigma^{-1} \mu^{\Dcal_0}_{\mathbf{M}} & \Sigma^{-1}
    \end{bmatrix} \begin{bmatrix}
    \mu^{\Dcal_0}_V \\
    \sigma_{\textbf{M}, V} + \mu^{\Dcal_0}_V \mu^{\Dcal_0}_{\mathbf{M}}
    \end{bmatrix} \\
    &= \begin{bmatrix}
        \mu^{\Dcal_0}_V - (\mu^{\Dcal_0}_{\textbf{M}})^\top \Sigma^{-1} \sigma_{\textbf{M}, V} \\
        \Sigma^{-1} \sigma_{\textbf{M}, V}
    \end{bmatrix}.
\end{align*}
Here, the first element of the block matrix is the intercept term and the second term represents the regression coefficients for each of the regressors. Note that only the intercept depends on the mean $\mu^{\Dcal_0}$ (in linear regression, shifting the variables only changes the intercept).

Thus, in distribution $\Dcal_i$, the faithfulness condition is equivalent to that the intercept term changes. Under $\Dcal_i$, the observed data is $X_{\Dcal_i} = X_{\Dcal_0} + \Delta^{(i)}$, where $\Delta^{(i)}$ encodes the shifts due to the best response. By a similar analysis, the intercept term under $\Dcal_i$ from the regression of $V$ on $\mathbf{M}$ is
\begin{align*}
    (\beta^{\Dcal_i}_{V | \textbf{M}})_0 &= (\mu^{\Dcal_0}_V + \Delta^{(i)}_V) - (\mu^{\Dcal_0}_{\textbf{M}} + \Delta^{(i)}_{\textbf{M}})^\top \Sigma^{-1} \sigma_{\textbf{M}, V} \\
        &= (\beta^{\Dcal_0}_{V | \textbf{M}})_0 + \Delta^{(i)}_V - (\Delta^{(i)}_{\textbf{M}})^\top \Sigma^{-1} \sigma_{\textbf{M}, V}.
\end{align*}
For the faithfulness condition to hold, we must have $(\beta^{\Dcal_0}_{V | \textbf{M}})_0 \neq (\beta^{\Dcal_i}_{V | \textbf{M}})_0$ and therefore we have
\begin{equation}\label{eqn:faithfulness}
    \Delta^{(i)}_V \neq (\Delta^{(i)}_{\textbf{M}})^\top \Sigma^{-1} \sigma_{\textbf{M}, V}.
\end{equation}
\end{proof}

At a high level, the faithfulness assumption ensures that the conditional distribution does change, i.e., the interventional values, which are a function of the costs, do not precisely satisfy the linear relationship above. Using the form of the best response, we have equality in Equation~\ref{eqn:faithfulness} if some linear combination of $\{1 / c_j\}_{j=1}^n$, whose weights are a function of the SCM parameters and variances, is zero. This suggests that equality happens with
for a set of parameter values of measure zero, 
assuming $c_j$'s are drawn independently from some product distribution.

\begin{algorithm}[t]
\caption{Subroutine for finding PC-set of all nodes~\citep{aliferis2010local}}
\label{alg:pc_set_alg}
\begin{algorithmic}[1]
\State $S_{all} = \{X\}_{i=1}^n \cup \{Y\}$
\For{each node $V$}
    \State Initialize $MB = \emptyset$
    \For{each node $V' \in S_{all} \setminus \{V\}$}
        \If{\textbf{not} $V' \indep V | MB$}
            \State $MB.add(V')$
        \EndIf
    \EndFor
    \For{$V' \in MB$}
        \If{$V' \indep V | MB$}
            \State $MB.remove(V')$
        \EndIf
    \EndFor
    \For{$V' \in MB$}
        \If{\textbf{for} all subsets $s$ of $MB$, \textbf{not} $V \indep V' | s $}
            \State Set $V$ is adjacent to $V'$
        \EndIf
    \EndFor
\EndFor
\end{algorithmic}
\end{algorithm}

\newpage

\section{Deferred Material on Algorithm~\ref{alg:linear_alg}}\label{sec:lin_cost_proofs}

\subsection{Linear Cost Best Response}

Under this setting, an individual is optimizing:

\begin{equation}
\label{eqn:alg2_br}
\begin{aligned}
\max_{a} \quad & w^T(x + Ba)\\
\textrm{s.t.} \quad & \sum_{i=1}^n c_i |a_i| \leq b.   \\
\end{aligned}
\end{equation}

Then the optimal intervention $a^*(w)$ is as follows: with $i^* = \argmax_{j \in [n]} \frac{|(B^Tw)_j|}{c_j}$:

\begin{equation}\label{eqn:all_linear_opt}
a^*(w) = \sign((B^Tw)_{i^*}) \left[\frac{b}{c_{i^*}} e_{i^*}\right].
\end{equation}

From this, we observe that at most $2n$ types of interventions may be induced: $\pm \frac{b}{c_i} e_i$. Moreover, each one intervention can be induced. For example, we note that $a^*(w) = \frac{b}{c_{i}} e_{i} $ for $w = (B^T)^{-1} e_i$.

Now to address the tie-breaker in the case when some features can be immutable, let $S_M \subseteq [n]$ be the subset of features which are mutable. That is, $c_i \neq \infty \Leftrightarrow i \in S_M$. We will assume that if $w$ is such that $(B^Tw)_j = 0$ for all $j \in S_M$, then the optimal $a^*(w)$ we observe will be some intervention $i \in S_M$. We will make no assumption on how this tie-breaking is done and which index $i \in S_M$ is chosen, just that the tie-breaking is done the same way across individuals.


\subsection{Proof of Theorem~\ref{thm:linear_alg_thm}}

\begin{theorem}
Algorithm~\ref{alg:linear_alg} computes the Pareto-Frontier using at most $2n$ deployments.
\end{theorem}
\begin{proof}

We will prove algorithm correctness in several parts:

\paragraph{Estimation of $\Ecal_{\Dcal_0}[X]$} Through the distribution induced by $w$, we may observe $\Ecal_{\Dcal_0}[X] + Ba^*(w)$. The first step of the problem is to estimate $\Ecal_{\Dcal_0}[X]$ such that we may observe $Ba^*(w)$ directly. 

To do this, we deploy $w = e_1$, $w= -e_1$. It remains to argue that $a^*(w) = -a^*(-w)$. This follows because $i^* = \argmax_{j \in [n]} \frac{|(B^Tw)_j|}{c_j} \Leftrightarrow i^* = \argmax_{j \in [n]} \frac{|(B^T(-w))_j|}{c_j}$. From this, we can conclude $a^*(w) = \sign((B^Tw)_{i^*}) \left[\frac{b}{c_{i^*}} e_{i^*}\right] = -\sign((B^T(-w))_{i^*}) \left[\frac{b}{c_{i^*}} e_{i^*}\right] = -a^*(-w)$, using the closed form optimal solution in Equation~\ref{eqn:all_linear_opt}.

\paragraph{Elicitation of all possible distributions} WLOG $S_M = \{1, ..., k\}$, where $k \leq n$ is the number of mutable features. Let $W^0$ denote the nullspace of $[B e_1; ...; B e_k]$, where $e_i$ corresponds to the standard basis vector wrt node $i$.

We will show that, after $n - 1$ iterations of for-loop~\ref{alg_line:all_linear_loop}, we will not have observed a new distribution (corresponding to a new underlying intervention) $n - k$ times (reaching Condition~\ref{alg_line:non_observe_linear}). From this, we must have observed $ n - 1 - (n - k) = k - 1$ new underlying interventions, which must correspond to the rest of the $k - 1$ interventions that are possible. Thus, when the algorithm terminates, we would have observed all $2k$ distributions that are possible, corresponding to the $k$ possible interventions, with both signs possible for each intervention.

Consider iteration $i$ and suppose we have observed distributions corresponding to interventions on nodes $\{i_1, ..., i_{k'}\}$ for $k' < k$. The algorithm uses SVD to find a vector $w_i \neq 0$ in the null-space of $W$, which means it is also in the nullspace of $[B a^*(w'_1); ...; B a^*(w'_{k'})]$ where $w'_j$ denotes the model that induced intervention $i_j$. Note that since $\text{rank}(W) \leq i < n$, the nullspace of $W$ is non-empty and we can always find such a $w_i$.

Next, notice that since $w_i$ is in the nullspace of $[B a^*(w'_1); ...; B a^*(w'_{k'})]$, it must also be in the null-space of $[B e_{i_1}; ..., Be_{i_{k'}}]$. With this, $w_i$ must be such that $w_i^T B e_{i_j} = 0 \Leftrightarrow (B^Tw_i)_{i_j} = 0$ for all $j \in [k']$. Moreover, we know that since $B^T$ is full-rank, $B^Tw_i \neq 0$. And so, if $w_i^T B e_j \neq 0$ for some $j \in [k] \setminus \{i_1, ..., i_{k'}\}$, then we will observe a new distribution corresponding to some intervention in $[k] \setminus \{i_1, ..., i_{k'}\}$.

If it is the case that we do not observe a new distribution, we must have that $w_i^T B e_j = 0$ for all $j \in [k] \setminus \{i_1, ..., i_{k'}\}$ as well. Therefore, $w_i \in W^0$.

Suppose by contradiction, we reach Condition~\ref{alg_line:non_observe_linear} more than $n - k$ times. This means that there exists at least $n - k + 1$ vectors $w_{j_1}, .., w_{j_{n - k +1}}$ in $W^0$. By construction, each vector is orthogonal to the rest, which means $w_{j_1}, .., w_{j_{n - k +1}}$ are linearly independent. This implies that $\dim(W^0) \geq n - k + 1$. 

This however is a contradiction, because we have:
$$ \dim(W^0) + \dim(\{Be_1, ..., Be_k\}) = \dim(W^0) + k = n,$$
since $B$ is full rank and its columns are linearly independent.

\paragraph{Optimization of the Tradeoff} Finally, we note that since there are $2k$ distributions that may be induced, there are $2k$ possible (improvement, risk) pairs that can form the Pareto frontier.

Given a distribution $\Dcal_i$, we know that the improvement is fixed. Thus, it remains to evaluate the best attainable MSE $R_i$ under this distribution. To do this, we solve:

\begin{equation}
\begin{aligned}
\min_{w} \quad & R_{\Dcal_i}(w)\\
\textrm{s.t.} \quad & w^T (\Ecal_{D_i}[X] - \Ecal_{\Dcal_0}[X]) \geq w^T (\Ecal_{D_j}[X] - \Ecal_{\Dcal_0}[X]), \forall j \neq i.  \\
\end{aligned}
\end{equation}

Let the optimal intervention underlying $\Dcal_i$ be $a_{i^*}$. Observe that $w$ induces intervention $a_{i^*}$ iff $w^T (\Ecal_{D_i}[X] - \Ecal_{\Dcal_0}[X]) \geq w^T (\Ecal_{D_j}[X] - \Ecal_{\Dcal_0}[X]), \forall j \neq i$ since:

$$w^T (\Ecal_{D_i}[X] - \Ecal_{\Dcal_0}[X]) \geq w^T (\Ecal_{D_j}[X] - \Ecal_{\Dcal_0}[X]) \Leftrightarrow w^T(B a_{i^*}) \geq w^T(B a_j).$$

Therefore, solving this optimization program finds the model with the lowest MSE on $\Dcal_i$, out of all models that lead to $a_{i^*}$ being the best response that induces interventional distribution $\Dcal_i$.

Computing this for each possible $\Dcal_i$ generates the $2k$ possible (improvement, risk) pairs that can form the Pareto frontier. Hence, to obtain the Pareto frontier, it only remains to retain all undominated (improvement, risk) pairs out of the $2k$ pairs (which is done in the last step of the algorithm) as only these pairs will form the Pareto front.

\end{proof}

\newpage

\section{Deferred Material on Algorithm~\ref{alg:general_alg}}\label{sec:gen_cost_proofs}

\subsection{Algorithm Assumptions}

\textbf{Cost Function Property:} To recap, for cost functions under Cost Function Class~\ref{cost_class:general}, we no longer require that 
$\frac{\partial c(a; x)}{\partial a_i} \Bigr|_{a_i=0} = 0$ as in Cost Function Class~\ref{cost_class:quadratic}.

Moreover, we allow $S_i$ to be any subset of $\{x_j\}_{j=1}^n$. This class of cost functions includes the homogeneous quadratic cost function, $c(a) = \frac{1}{2} \|a\|^2$~\citep{shavit2020causal}
and the homogeneous linear cost function, $c(a) = \sum_{i=1}^n c_i |a_i|$~\citep{bechavod2020causal, kleinberg2020classifiers}
considered in prior works.

\textbf{Applicable SCM:} The discovery algorithm that we develop 
applies to a large family of causal graphs: 
Additive Noise Models~\citep{peters2017elements}.

\subsection{Algorithm Proof}

The algorithm relies on the following observation:

\begin{proposition}
Deploying $X_i - g_i(X_{\pa(i)})$ induces an intervention only on node $X_i$.
\end{proposition}~\label{prop:additive_mech}
\begin{proof}
With this choice of scoring mechanism, the optimization objective for an individual with features $x$ is:
\begin{equation}
\begin{aligned}
\max_{a} \quad & x'_i - g_i(x'_{\pa(i)})\\
\textrm{s.t.} \quad & x'_j = g_j(x'_{\pa(j)}, u_j) + a_j \quad \forall j \in [n]   \\
\quad & c(a_1, ..., a_n; x) \leq b.
\end{aligned}
\end{equation}

Since $g_i$ is additive, we may plug in $g_i(x'_{\pa(i)}, u_i) = g_i(x'_{\pa(i)}) + u_i$ and the objective becomes $x'_i - g_i(x'_{\pa(i)}) = g_i(x'_{\pa(i)}, u_i) + a_i -   g_i(x'_{\pa(i)}) = u_i + a_i$. And so, each individual with features $x$ is optimizing:

\begin{equation}
\label{eqn:alg3_br}
\begin{aligned}
\max_{a} \quad & u_i + a_i\\
\textrm{s.t.} \quad & x'_j = g_j(x'_{\pa(j)}, u_j) + a_j \quad \forall j \in [n]   \\
\quad & \sum_{j=1}^n c_j(a_j; x) \leq b.
\end{aligned}
\end{equation}

Note that $u_i$ is a fixed constant. Moreover, since each cost function $c_j$ is strictly increasing in the magnitude of $a_j$, we must have that $a^*_j = 0$ for $j \neq i$ (otherwise one can increase $a_i$ instead to increase the objective). And so, we have that only $a^*_i \neq 0$. 
\end{proof}
\textbf{Remark:} In the linear case, this choice is in fact the unique policy that produces an incentive to \emph{only} invest in $X_i$.

\subsection{Proof of Correctness of Algorithm~\ref{alg:general_alg}}

We first describe the mild faithfulness assumption we will need.


\textbf{Assumption (Mean Interventional Faithfulness):} Let $V \in G$ be any node in $G$. Let $\Ical^i$ be the set of all non-$\verb|null|$, intervention nodes in the augmented graph $\tilde{G}$ corresponding to $\Dcal_i$, then:
$$\exists I \in \Ical^i \text{ s.t } I \notindep  \tilde{V} \Leftrightarrow  \Ecal_{\Dcal_i}[V] \neq \Ecal_{\Dcal_0}[V].$$

For the proof below, we will actually only require a particular instantiation of the faithfulness assumption above (as used also e.g by~\citet{zhang2021matching}). This particular case is that if node $X_i$ is intervened upon, and $V$ is its child highest in the topological order, then the mean of $V$ in the interventional distribution shifts. Put another way, this assumes that the interventional values on $X_i$ and $V$ as well as the SCM parameter relating the two nodes are not such that the interventions cancel out exactly, and the mean of $V$ does not change.

\begin{theorem}
Algorithm~\ref{alg:general_alg} recovers the full graph structure using at most $n(n-1) / 2$ deployments.
\end{theorem}

\begin{proof}

We will prove that the parents of each node are correctly identified by the algorithm, which implies that the full graph is correctly identified. To do this, we will show that Algorithm~\ref{alg:general_alg} always maintains the invariant property (1) that, each iteration, the node that is added to $S$ from the subgraph $SG$ is always a leaf node.

(1) has the implication that (2) no node is added before all of its descendants. Indeed, a node is only added when it is a leaf, and if a node does have at least one descendant in the subgraph, it is not a leaf and cannot be added. 

Thus, with (1), the algorithm will be such that the following holds: (3) that every node in $S$ has its parents correctly and completely identified. When a new node is added to $S$, we identify all nodes in $SG$ adjacent to the new node as its parents. Since the node is a leaf, every such node in $SG$ can only be its parents, and by (2) must be all of its parents. And so, this ensures that this new node's parents also satisfy (3).

We see that (1) is satisfied for $S$ at initialization. To prove (1) always holds, it suffices to show that leaf nodes in any subgraph $SG$ will be such that Condition~\ref{alg_line:general_test} is always false and no non-leaf node in any subgraph $SG$ will be such that Condition~\ref{alg_line:general_test} is always false.

Let $\Xcal$ be the set of $\{X_i\}_{i=1}^n$ nodes. For a particular subgraph $SG$, suppose $SG$ and $S = \Xcal \setminus SG$ satisfies (2).

\textbf{Non-Leaf Nodes in $\Xcal$ do not pass test:}
First, note that there has to exist at least one node that is intervened upon. This is because $f'$ is monotonically increasing in $a_i$, so the best response will include non-zero interventions on at least one node.

Next, since the policy is a function of $X_{P_i}$ and $X_i$, the intervention will take place on only node(s) that are ancestors of nodes of $\{i\} \cup P_i$. This is again because for any $j \not \in \anc(i) \cup \anc(P_i)$, changing $a_j$ will not change $\tilde{X}_{P_i}$ and $\tilde{X}_i$ (and thus the objective), but strictly increases costs. By (2), since $\{i\} \cup P_i \in SG$, $\anc(i) \cup \anc(P_i) \in SG$. That is, every node that will be intervened upon when $f_i = X_i - \hat{g}_i(X_{P_i})$ is deployed will be in the subgraph $SG$.

Out of all nodes which are intervened upon under $f = X_i - \hat{g}_i(X_{P_i})$, let $k$ be the index of a node such that none of its ancestors is intervened upon (i.e an intervened node that is highest in topological order). If $k \neq i$, we have $\Ecal_{\Dcal_i}[X_k] \neq \Ecal_{\Dcal_0}[X_k]$ since $X_k$ is dependent on $I_k^i$.

Else, we have that $k = i$. We know that since $X_i$ is not a leaf, it must have at least one child. Let $V$ in $SG$ be the child of $i$ with the highest topological order. We have that $V$ is dependent on $I_i^i$ due to chain $I_i^i \rightarrow X_i \rightarrow V$. And so, by our faithfulness assumption, its expectation under $\Dcal_i$ will change due to the intervention on $X_i$. Note that it may be that under $f$, $V$ may also be intervened upon; our faithfulness assumption is that the SCM parameters are not such that the two interventions cancel out exactly.

Either way, we conclude that Condition~\ref{alg_line:general_test} will hold for at least one node in the subgraph.

\textbf{Leaf Nodes in $\Xcal$ pass test:} Suppose first that subgraph $SG$ has a leaf node $X_i$. Then, all its parents must be still in the subgraph $SG$ by property (2). Since it is a leaf in $SG$, none of its children is in $SG$. And so, all the nodes $P_i$ adjacent to $X_i$ in the $GS$ must be its parents and only its parents. Thus, in additive SCMs, the model $\hat{g}_i = (X_i | X_{P_i})_{\Dcal_0}$ identifies $g_i$, the true SCM parameter, up to a fixed constant which does not affect the best response. Thus, from Proposition~\ref{prop:only_intervention}, we have that in $\Dcal_i$, only $X_i$ is intervened upon. 

With this, we can conclude that no other node in $SG$ has its distribution change since $X_i$ has no descendants; $X_i$'s intervention only changes the distribution of $X_i$. Hence if $X_i$ is a leaf, Condition~\ref{alg_line:general_test} will always be false. (1) will be satisfied as we have just shown that a node of the subgraph will make Condition~\ref{alg_line:general_test} always false iff it is a leaf.

\textbf{Lone $Y$ leaf:} Finally, the remaining case is when $Y$ is the only leaf of the current subgraph. We have just shown that if there is a leaf in $SG$ and in $\Xcal$, it will meet the criteria. We have also shown earlier that no non-leaf node in $\Xcal$ can meet the criteria. So if it is the case that no nodes in $\Xcal$ meets the criteria, then by the process of elimination, $Y$ must be the only leaf in the subgraph.

\textbf{Termination:} The algorithm terminates when there is only one node left in the subgraph. By (2), it must be a root node in the full graph. This means that we have also managed to identify its parents, which is the empty set.

\textbf{Complexity:} The algorithm adds one node to $S$ per iteration and there are at most $n$ iterations. During each iteration, we run at most $|SG|$ regressions and $|SG|$ deployments. And so, at most $n(n-1) / 2 $ regressions and deployments are needed to discover the graph.

Note that a quadratic number of deployments is needed since, unlike Algorithm~\ref{alg:quadratic_alg}, the costs may be such that we are only able to guarantee an intervention on a node $X_i$ when we have correctly guessed its set of parents.

\end{proof}

\textbf{Remark:} Intuitively, leaves of subgraphs are useful since intervention and change in distribution is isolated to the leaf nodes. By contrast, for root nodes, interventions will change nodes of the entire subgraph.

\newpage

\section{Relaxation of Assumptions on inputs to the Algorithms}\label{sec:assumption_justification}

\subsection{Graph Skeleton}

The assumption of knowledge of the graph skeleton is mild, as it may be readily computed using a variety of algorithms. These include the first part of the PC algorithm or local discovery Algorithm~\ref{alg:pc_set_alg}, which identifies a node's Parent-Child set (PC-set) and has its complexity scale with the node with the highest degree in the undirected graph.

Also, we note that one may modify Algorithm~\ref{alg:general_alg} slightly to bypass this assumption. Instead, in order to perform discovery, one would require access to the following oracle.

\textbf{Assumption (Parent-Extraction Learning Oracle):} Given node $X_i$ and $\Dcal_0$, as well as a set of nodes $S$ such that (1) $\pa(S) \subseteq S$ and (2) none of $X_i$'s descendants is in $S$, returns $X_{\pa(i)}$ and $g_i = (X_i | X_{\pa(i)})_{\Dcal_0}$.

In the case of Linear SCMs, this oracle may be implemented by LASSO and by looking at which features have non-zero coefficients in the model learned using LASSO. However, for nonlinear SCMs, it is unclear if there is a learning algorithm that would be guaranteed to prune out non-parents, especially if there are multiple local optima.

For Algorithm~\ref{alg:general_alg}, it suffices to just alter Line~\ref{alg_line:general_regression_args} to regress $X_i$ against $X_{SG}$ using following oracle:

\begin{itemize}
    \item If it is a leaf, we will obtain $X_{\pa(i)}$ and $g_i = (X_i | X_{\pa(i)})_{\Dcal_0}$. And as we have shown previously, deploying $f_i$ using $\hat{g}_i = (X_i | X_{\pa(i)})_{\Dcal_0}$ will lead to Condition~\ref{alg_line:general_test} being always false iff it is a leaf of $SG$. 

    \item For the orientation step (Line~\ref{alg_line:general_orientation}), we may orient its edges with its parents that are returned by the oracle. Finally, in the case of $Y$ being the leaf, we will again use the oracle to regress $Y$ against $SG$. Since it is a leaf, we will obtain $\pa(Y)$ from the oracle, which we may orient accordingly.

\end{itemize}

\subsection{Natural Distribution}

In this subsection, we discuss the assumption of access to the natural distribution, as assumed in e.g~\citep{shavit2020causal}. As in~\citep{shavit2020causal}, this distribution may be induced by the null mechanism. Alternatively, for cases where the best response is an odd function (e.g when the objective is an odd function and the cost is an even function), the natural distribution may be obtained by one further deployment, since the best response will be an odd function. For instance, in the linear SCM, quadratic cost case, $a^*(w) = -a^*(-w)$. We may then deploy $f = -x_1$, from which obtain natural distribution $\Dcal_0 = (\Dcal_1 + \Dcal_{-1}) / 2$.

\newpage

\section{Miscellaneous}

\textbf{Other SCM classes under Cost Class~\ref{cost_class:general}:} We note that Algorithm~\ref{alg:general_alg} may be also be applied to the Multiplicative Noise Models. That is, $X_i = g_i(X_{\pa(i)})U$.

For the algorithm to succeed, we will need to additionally assume that $g_i$ is monotonically increasing, which we may learn using monotonic functions $\hat{g}$ (e.g~\citep{kleinberg2020classifiers}), and that $a \geq 0$. Then Algorithm~\ref{alg:general_alg} will again work with the choice of model $f_i$ being: $f_i = \frac{X_i}{\hat{g}_i(X_{P_i})}$. 

It remains an open question whether one can develop a general algorithm that works for even more general SCM classes, such as the Post-Nonlinear Models~\citep{zhang2012identifiability}.

\textbf{Causal Side-Information:} An astute observer may notice that since we have knowledge of the cost function $c(a_i; \tau_i)$, $\tau_i$ may inform which nodes are upstream of $i$. We wish to note our algorithm is designed to handle settings in which $\tau_i = \emptyset$ (with no causal side-information revealed) and does not leverage this causal knowledge. Indeed, to reiterate, our algorithms do not assume access to the cost function, but rather assume knowledge that they belong to a certain cost class. Furthermore, there exists a large class of cost functions where this causal side-information is not enough to orient the full graph.